\documentclass{article} 
\usepackage{iclr2027_conference,times}

\usepackage{amsmath,amsfonts,bm}

\def\eqref#1{equation~\ref{#1}}

\def\1{\bm{1}}

\DeclareMathAlphabet{\mathsfit}{\encodingdefault}{\sfdefault}{m}{sl}
\SetMathAlphabet{\mathsfit}{bold}{\encodingdefault}{\sfdefault}{bx}{n}

\usepackage{hyperref}       
\usepackage{booktabs}       
\usepackage{amsfonts}       
\usepackage{nicefrac}       
\usepackage{microtype}      
\usepackage{xcolor}         
\usepackage{wrapfig}
\usepackage{bm}
\usepackage{amsmath}
\usepackage{amssymb}
\usepackage{mathtools}
\usepackage{amsthm}
\usepackage{float}
\usepackage{wrapfig}
\usepackage{makecell}
\usepackage{amsthm}
\usepackage{comment}
\usepackage{refcount}
\usepackage{graphicx}
\usepackage{multirow}
\usepackage{amsmath, amssymb}
\usepackage[utf8]{inputenc} 
\usepackage[T1]{fontenc}    
\usepackage{url}            
\usepackage{booktabs}       
\usepackage{amsfonts}       
\usepackage{nicefrac}       
\usepackage{microtype}      
\usepackage{xcolor}         
\usepackage{bbm}
\usepackage{subcaption}

\newcolumntype{?}{!{\vrule width 1pt}}
\newtheorem{theorem}{Theorem}[section]

\newtheorem{lemma}[theorem]{Lemma}

\usepackage[textsize=tiny]{todonotes}

\title{Anti-coherence Regularized Absolute Scaling for Out-of-Distribution Detection}

\author{Donghwan Kim \\
  Department of Industrial Engineering\\
  Yonsei University\\
  Seoul, South Korea \\
  \texttt{cain0709@yonsei.ac.kr} \\
  \And
  Youngwoong Jin \\
  Department of Intelligent Data and Optimization\\
  Yonsei University\\
  Seoul, South Korea \\
  \texttt{duddnd1113@yonsei.ac.kr} \\
    \And
  Hyunsoo Yoon$^\dagger$ \\
  Department of Industrial Engineering\\
  Yonsei University\\
  Seoul, South Korea \\
  \texttt{hs.yoon@yonsei.ac.kr} \\
}

\iclrfinalcopy 
\begin{document}

\maketitle

\begin{abstract}
Out-of-distribution (OOD) detection is recognized as a critical problem in AI safety. SCALE is an effective post-hoc OOD detection method in the activation-shaping family that requires no training data and rescales penultimate activations using a percentile-based magnitude ratio. However, this ratio is ill-defined when negative activations are present, since positive and negative contributions cancel within the underlying sum, causing SCALE to degenerate on backbones such as ViT. We address this with an \textbf{Absolute Fix} that computes the ratio over absolute-valued activations rather than discarding negative activations through ReLU. Building on this corrected scale, we propose \textbf{Anti-coherence} regularization, which penalizes the scaling factor by the fraction of a sample's max-logit contributions to the max logit that are sign-mismatched and thereby suppresses unreliable evidence for the predicted class. Despite requiring \textbf{no access to training data or training labels}, our method \textbf{ties or outperforms} activation-shaping methods that rely on training data in terms of AUROC and FPR@95 on CIFAR-10 and ImageNet under the OpenOOD v1.5 protocol.
\end{abstract}

\section{Introduction}
\label{sec:intro}
A fundamental limitation of deep neural networks stems from the closed-world assumption under which they are trained: once deployed, these models frequently encounter inputs that deviate from the training distribution while continuing to produce confident predictions despite this mismatch. Such failures are consequential in safety-critical domains such as robotics \citep{isaku2025out}, cyber security \citep{corsini2023existing}, and medical imaging \citep{hong2024out}, where an undetected distribution shift can lead to severe real-world harm. Out-of-distribution (OOD) detection has therefore emerged as a key safeguard that flags unfamiliar inputs before they translate into erroneous decisions, supporting the broader goal of building reliable and trustworthy AI \citep{yang2024generalized}.

Existing OOD detection methods can be broadly divided into two categories. The first category consists of training-based methods incorporating synthetic or real OOD data or auxiliary objectives into training to encourage better separation between in-distribution (ID) and OOD samples \citep{ming2023how, du2022vos, zhu2023diversified, tao2023nonparametric}. However, practical access to OOD datasets is often restricted and the additional training adds computational cost that renders these methods impractical for real-world deployment \citep{yang2024generalized}. The second category consists of post-hoc methods, which operate on a frozen pretrained feature extractor without updating its weights. While some rely on statistics computed from training data or additional gradient computation, none require retraining the model or auxiliary OOD datasets.

Within post-hoc approaches, methods can be further categorized according to what they utilize to compute the OOD score. Logit-based methods \citep{hendrycks2017a,hendrycks2022scaling} utilize the values of the output logits to compute the OOD score, representation-based methods \citep{ren2021simple,lee2018simple,sun2022out} leverage statistics such as distance estimated from penultimate layer activations, and activation-shaping methods \citep{sun2021react, xu2023vra} manipulate the penultimate layer activations themselves prior to scoring. Among these, activation shaping has emerged as a promising direction, as it often requires no ID statistics and is implemented through simple operations \citep{regmi2026adascale}, such as scaling, rectifying, or pruning (e.g., \citet{sun2021react, djurisic2023extremely, xu2024scaling}), which are then combined with another post-hoc scoring function such as Energy \citep{liu2020energy} to compute the final OOD score.

The activation-shaping method SCALE~\citep{xu2024scaling} defines a scaling factor as the ratio between the sum of all penultimate activations and the sum of those exceeding a given percentile and applies this factor uniformly across the entire penultimate layer to achieve strong OOD detection performance. However, this percentile-based scaling has been reported to suffer substantial performance degeneration on architectures with negative penultimate activations such as ViT \citep{dosovitskiy2021an} and Swin Transformer \citep{liu2021swin}, whereas it remains stable on architectures such as ResNet \citep{he2016deep} whose penultimate layer follows a ReLU and therefore contains no negative activations \citep{regmi2026adascale}. To address this architecture-dependent limitation, \citet{djurisic2025logit} explicitly apply ReLU to the penultimate activations before computing the scaling factor and show that removing negative activations in this way improves performance. Yet this approach discards all information carried by negative activations, and thus a limitation remains in that it fails to fully reflect the entire penultimate layer when computing the scaling factor. Since this leaves room for improvement, we propose the \textbf{Absolute Fix}, which applies the absolute value function to the penultimate activations so that negative activations are also reflected in the scaling factor and yields more consistent performance gains than the ReLU-based rectification.

Building on this corrected scale, we decompose each sample's contribution to the max-logit and use the fraction of this contribution that drives the max-logit downward, which we term \textbf{Anti-coherence}, to regularize the scaling factor. In real-world deployment, training data is often inaccessible due to privacy constraints \citep{NEURIPS2025_ee4432ad}, and our method requires no access to it. Despite this, it \textbf{ties or outperforms} activation-shaping methods that rely on training data or additional gradient computation under the OpenOOD v1.5 \citep{zhang2024openood} protocol on CIFAR-10 \citep{krizhevsky2009learning} and ImageNet \citep{deng2009imagenet} in-distribution settings. In particular, our method achieves consistent performance across a diverse set of backbones on ImageNet.

In summary, our contributions are threefold.

\begin{itemize}
    \item We propose the \textbf{Absolute Fix}, which computes SCALE's scaling factor over absolute-valued penultimate activations rather than through ReLU-based rectification, thereby preserving the magnitude of negative activations and yielding consistent performance gains across architectures with negative penultimate activations.
    \item We introduce \textbf{Anti-coherence}, a regularization term derived from decomposing each sample's contribution to the maximum logit, which regularizes the scaling factor according to the fraction of this contribution that drives the maximum logit downward.
    \item We validate our method under the OpenOOD v1.5 protocol on CIFAR-10 and ImageNet, demonstrating that it ties or outperforms activation shaping methods that rely on training data or additional gradient computation, despite requiring neither.
\end{itemize}

\section{Related Works}
\label{sec:related_works}
\textbf{Activation Shaping Methods} operate by directly manipulating penultimate layer activations. ReAct \citep{sun2021react} truncates activations above a percentile threshold determined from training data, addressing the overconfidence exhibited by OOD samples. VRA \citep{xu2023vra} further truncates a lower bound, thereby widening the separation between ID and OOD activations. DICE \citep{sun2022dice} instead truncates the weights connected to these activations, while ASH \citep{djurisic2023extremely} combines truncation with activation scaling to improve detection performance. \citet{xu2024scaling} subsequently show that truncation degrades performance and achieve improved detection through scaling alone. OptFS \citep{zhao2024towards} estimates an optimal feature scaling vector from training data, and BATS \citep{zhu2022boosting} and LAPS \citep{he2024exploring} instead define a typical set from the activation statistics of training samples and rectify test-time activations according to this typical set. Catalyst \citep{Hassan_2026_CVPR} and DAVIS \citep{hassan2026davis} observe that channel-wise statistics of the pre-pooling activation map, such as mean, standard deviation (or variance), and the dominant (maximum) activation, are discarded during global average pooling (GAP), and exploit these statistics to compute a scaling factor or enrich the feature vector. While DAVIS extends to hierarchical architectures with a GAP stage such as Swin Transformer and ConvNeXt, both methods presuppose a spatial pre-pooling activation map and are therefore not directly applicable to architectures such as ViT that lack an explicit GAP operation.

\textbf{Representation-based Methods} primarily exploit the geometric structure of the penultimate layer's feature vector, computing a distance metric for the OOD score. \citet{lee2018simple} propose the class-wise Mahalanobis distance, fitting class-conditional Gaussians from training data statistics of the penultimate layer. Subsequent work has refined this distance by incorporating additional global information or $\ell_2$ feature normalization \citep{ren2021simple, muller2025mahalanobis++, kim2026mahavar}. Other approaches instead measure distance non-parametrically, using the $k$-nearest-neighbor distance to training feature vectors \citep{sun2022out} or fusing such a distance with a logit-based score \citep{park2023nearest}. \citet{liu2024fast} use the distance between a test sample and the decision boundary as the OOD score, while \citet{liu2025detecting} exploit the Neural Collapse phenomenon \citep{papyan2020prevalence} observed in the penultimate layer of pretrained backbones. Additionally, \citet{wang2022vim} derive a virtual OOD-class logit from the residual of penultimate features outside the principal subspace of training features and score samples by the softmax probability assigned to it. CADRef \citep{ling2025cadref} partitions feature components by whether their weighted contribution to the max logit aligns in sign with the classifier weight, and computes the deviation from class-conditional training averages within each partition, normalizing the result by a logit-based score. All of these methods, however, require access to the training dataset, a requirement our method does not share.

\textbf{Logit-based Methods} derive the OOD score directly from the logit values. \citet{hendrycks2017a} propose the maximum softmax probability as a baseline OOD score. \citet{liu2020energy}, motivated by energy-based models from statistical physics, propose an energy score derived from the logits. \citet{Liu_2023_CVPR} introduce a generalized entropy score that captures small variations in the softmax probabilities overlooked by standard entropy. \citet{djurisic2025logit} scale the logits themselves rather than the penultimate activations and compute the scaling factor from ReLU-rectified activations to further improve performance.

\section{Absolute Fix: Handling Negative Activations in $Q/Q_p$}
\label{sec:abs}

\subsection{Preliminaries}
\label{subsec:preliminaries}
Let $\phi: \mathcal{X} \to \mathbb{R}^d$ denote the feature extractor that maps an input to its $d$-dimensional penultimate layer activation, where the input space $\mathcal{X} = \mathcal{X}_{\text{in}} \cup \mathcal{X}_{\text{OOD}}$ is partitioned into disjoint ID and OOD subsets $\mathcal{X}_{\text{in}}$ and $\mathcal{X}_{\text{OOD}}$, with $\mathcal{X}_{\text{in}} \cap \mathcal{X}_{\text{OOD}} = \emptyset$. We assume that $\phi(\cdot)$ is pretrained on samples drawn from $\mathcal{X}_{\text{in}}$ for a downstream task such as classification (e.g., pretrained on ImageNet). Let $\textbf{z} = \phi(\mathbf{x}) \in \mathbb{R}^d$ denote the penultimate activation, with components $z_1, \dots, z_d$. SCALE \citep{xu2024scaling} computes a scaling factor $r$ from the ratio between the sum of all components of $\textbf{z}$ and the sum of components exceeding the $p^{th}$ percentile of $\textbf{z}$, given by
\begin{align}
r = \frac{Q}{Q_p}, \quad \text{where} \quad Q = \sum_{i=1}^{d} z_i, \quad Q_p = \sum_{i \,:\, z_i > P_p(\textbf{z})} z_i,
\end{align}
and $P_p(\textbf{z})$ denotes the $p^{th}$ percentile value among the $d$ components of $\textbf{z}$. Intuitively, $r$ decreases when activation mass concentrates in a small number of large components and increases when it spreads across many, and \citet{xu2024scaling} report that $r$ is on average larger for ID samples than for OOD samples. This can be interpreted through \citet{sun2021react}, who observe that OOD samples tend to trigger disproportionately high activations at a small number of units. The activation is then rescaled by $\exp(r)$ and passed through the classifier's fully connected layer to obtain the logit $\boldsymbol{\ell}'$, from which the final OOD score $S(\mathbf{x})$ is computed using the energy score\footnote{The scoring function used here is not restricted to Energy, but we adopt it as it is conventionally used with activation shaping methods.} \citep{liu2020energy},
\begin{align}
\boldsymbol{\ell}' = \mathbf{W}\big(\exp(r) \cdot \textbf{z}\big) + \textbf{b}, \qquad S(\mathbf{x}) = T \cdot \log \sum_{k=1}^{C} \exp({\ell'_k/T}),
\end{align}
where $\mathbf{W} \in \mathbb{R}^{C \times d}$ and $\textbf{b} \in \mathbb{R}^C$ are the weight and bias of the fully connected layer for $C$ classes, and $\ell'_k$ is the $k$-th component of $\boldsymbol{\ell}'$. A test sample is classified as in-distribution when $S(\mathbf{x})$ exceeds a threshold $\tau$, and as OOD otherwise. Following standard practice, $\tau$ is fixed so that the true positive rate (TPR) on ID samples equals 95\%. Under this threshold, the false positive rate, denoted FPR@95, measures the proportion of OOD samples incorrectly classified as ID.

\begin{figure}[t]
\centering
\includegraphics[width=\textwidth]{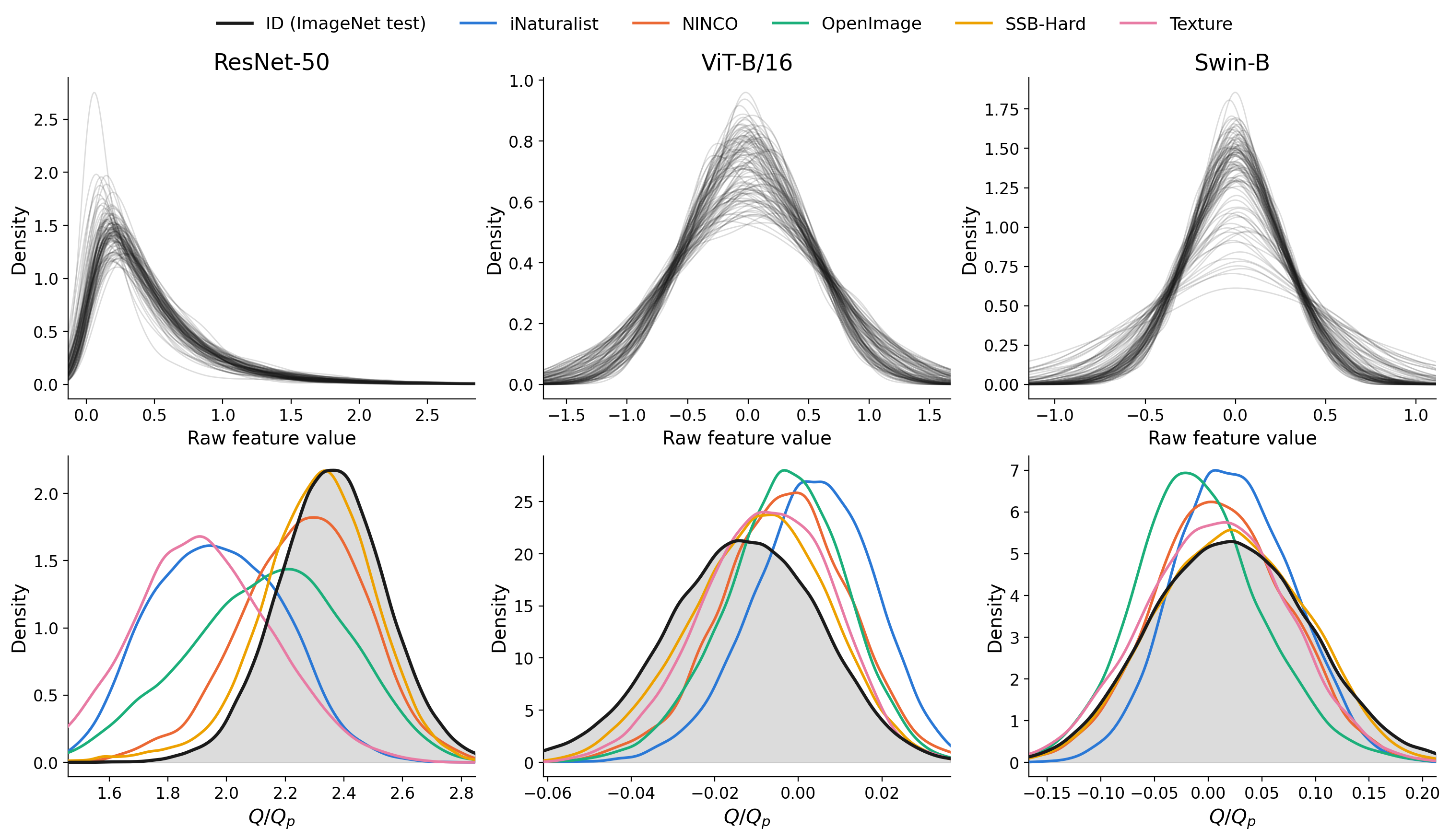}
\caption{Top: per-sample density of raw penultimate activation values for ID (ImageNet test) samples on ResNet-50, ViT-B/16 and Swin-B. ResNet-50, whose penultimate layer follows a ReLU, exhibits only non-negative activations whereas ViT-B/16 and Swin-B exhibit substantial negative activation mass. Bottom: distribution of $Q/Q_p$ for ID and OOD datasets on the same backbones. On ResNet-50, ID samples exhibit a visibly higher ratio than OOD samples, while on ViT-B/16 and Swin-B negative activations cause $Q$ and $Q_p$ to lose their intended meaning and this pattern no longer holds.}
\label{fig:raw_activation_main}
\end{figure}

\subsection{Absolute-Value Percentile Scaling}
\label{subsec:abs}
We now revisit $r$ more closely. Its computation relies on the $p^{th}$ percentile threshold $P_p(\mathbf{z})$ which is intended to isolate the components with the largest activation magnitude. When the penultimate activation is non-negative as in architectures such as ResNet where the penultimate layer follows a ReLU, magnitude coincides with value and $P_p(\mathbf{z})$ identifies these components as designed. However, for architectures such as ViT and Swin Transformer where the penultimate activation contains negative components, magnitude and value no longer coincide and summing over signed components allows positive and negative contributions to cancel. Consequently, $Q$ and $Q_p$ no longer reflect the intended notion of activation magnitude and the scaling factor $r$ deviates from its intended behavior. Following the OpenOOD v1.5 setting we set the ImageNet test split as ID and examine the distributions of $r$ for five OOD datasets in Figure~\ref{fig:raw_activation_main}. On ResNet, which always yields non-negative activations, $r$ is clearly larger for ID than for OOD. However, on ViT and Swin, whose activations are zero-centered by LayerNorm and therefore contain negative values, the empirical $r$ no longer separates ID from OOD and can even be larger for OOD.

This motivates the following observation, illustrated in Figure~\ref{fig:activation_logscale}. The figure shows the distribution of penultimate activations for ID and OOD samples on backbones with negative activations, where OOD exhibits a heavier tail than ID. This extends the single-sided heavy tail observed on ReLU-based backbones such as ResNet, where activations are strictly positive, into a tail that is heavy symmetrically on both sides of zero. To reflect this symmetric distribution when computing $r$, we propose the \textbf{Absolute Fix}, a simple improvement over the original formulation of $r$. This method computes $r$ using the absolute value of the penultimate activation at every step, denoted by the corrected ratio $r_{\text{abs}}$ and formulated as
\begin{align}
r_{\text{abs}} = \frac{Q_\text{abs}}{Q_{p,\text{abs}}}, \quad \text{where} \quad Q_{\text{abs}} = \sum_{i=1}^{d} |z_i|, \quad Q_{p, \text{abs}} = \sum_{i \,:\, |z_i| > P_{p,\text{abs}}(\mathbf{z})} |z_i|,
\end{align}
where $P_{p,\text{abs}}(\mathbf{z})$ denotes the $p^{th}$ percentile of $|\mathbf{z}|$. Note that when $\mathbf{z}$ is strictly non-negative, as in architectures such as ResNet and MobileNetV2, $r_{\text{abs}}$ reduces to $r$.

We analyze the behavior of $r$ and $r_{\text{abs}}$ across the six backbones, with full detail deferred to Appendix~\ref{appendix:absolute_fix}. On ViT-B/16, Swin-B, and ConvNeXt-B, $r_{\text{abs}}$ consistently separates ID from OOD in the intended direction. Despite lacking the symmetric heavy tail observed on these LayerNorm-based backbones due to its lower-bounded SiLU activation, EfficientNet-B1 still benefits from the Absolute Fix. On DeiT3-B the separation between ID and OOD visually worsens, while on BEiTv2-B $r_{\text{abs}}$ improves in the intended direction yet ID and OOD nearly coincide. Nonetheless both backbones show an improvement in OOD detection performance over the original SCALE and the ReLU-based fix. We provide a detailed discussion in Section~\ref{sec:exp} and Appendix~\ref{appendix:absolute_fix}.


\begin{figure}[t]
\centering
\includegraphics[width=\textwidth]{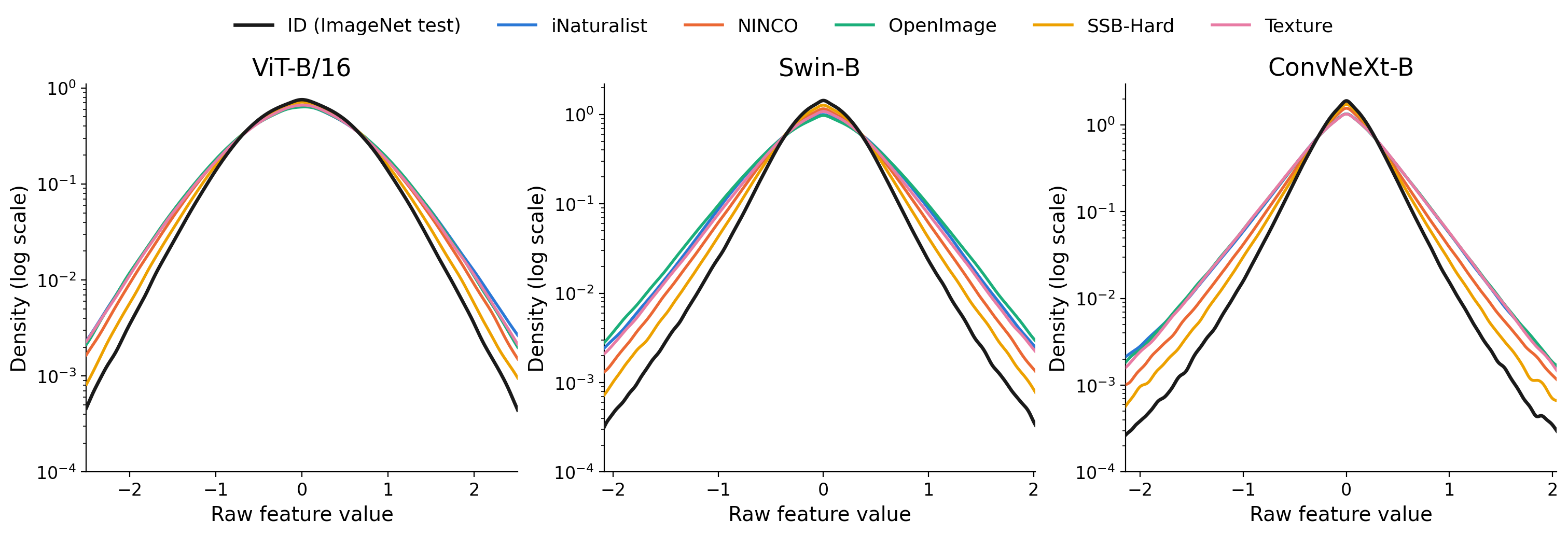}
\caption{Per-sample density (log scale) of raw penultimate activation values for ID and OOD samples on backbones with negative penultimate activations.}
\label{fig:activation_logscale}
\end{figure}

\section{Regularizing SCALE with Anti-coherence}
\label{sec:anti}

Activation-shaping methods are typically implemented by combining the shaped activation with a logit-based scoring function such as the energy score. We therefore first analyze how our primary object of study, SCALE, affects the energy score when combined in this manner. Throughout this analysis, $r$ refers to $r_{\text{abs}}=Q_{\text{abs}}/Q_{p,\text{abs}}$ as defined in the previous section. We first show, in Lemma~\ref{lem:maxlogit_energy}, that as the scaling factor increases, the max logit accounts for an increasingly large share of the energy score.

\begin{lemma}
\label{lem:maxlogit_energy}
Let $\mathbf{W} \in \mathbb{R}^{C \times d}$ be the classifier head's weight, and let $\ell_k := \mathbf{W}_k \cdot \mathbf{z}$ with $c^* := \arg\max_k \ell_k$ the unique maximizer, $\ell_{c^*} > 0$. For simplicity, we assume the bias is zero. The scaled logit and energy score are
\begin{align}
\boldsymbol{\ell}' = \mathbf{W}\big(\exp(r)\cdot\mathbf{z}\big), \qquad S(\mathbf{x}) = T\cdot\log\sum_{k=1}^{C}\exp(\ell'_k/T),
\end{align}
which decomposes as
\begin{align}
S(\mathbf{x};r) = \ell'_{c^*}(r) + \xi(r), \qquad \xi(r) := T\log\Big(1+\sum_{k\neq c^*}\exp\big(\exp(r)(\ell_k-\ell_{c^*})/T\big)\Big) \ge 0,
\end{align} 
where $\xi(r)$ is the contribution of the non-maximal logits. Let $\rho(r) := \ell'_{c^*}(r)/S(\mathbf{x};r)$ denote the fraction of the energy score attributable to the max logit. Then $\rho(r)$ is non-decreasing in $r$ with $\rho(r)\to1$ as $r\to\infty$.
\end{lemma}

Lemma~\ref{lem:maxlogit_energy} implies that as the scaling factor grows, the remaining terms of the energy score contribute progressively less. To verify this empirically, we compute the mean value of $\rho$, the fraction of the energy score attributable to the max logit, on ID and OOD ImageNet data under the OpenOOD v1.5 protocol; the results are reported in Table~\ref{tab:maxlogit_energy_ratio}. We find that after applying $r_{\text{abs}}$ the max logit term dominates the energy score for both ID and OOD samples across all backbones. This motivates us to reduce our subsequent analysis from the energy score itself to the max logit term alone. We therefore analyze how the max logit $\ell_{c^*}$ is affected by scaling.

Recall from the previous section that $\mathbf{W} \in \mathbb{R}^{C \times d}$ denotes the classifier head's weight matrix and let $\mathbf{W}_{c^*} \in \mathbb{R}^d$ denote its $c^*$-th row, the weight vector of the predicted class $c^*$. We decompose the contribution to the max logit prior to scaling as:
\begin{align}
\ell_{c^{*}} = C_p - C_n, \qquad C_p := \sum_{j:\, \mathbf{W}_{c^*,j}z_j>0} \mathbf{W}_{c^*,j}z_j, \qquad C_n := \sum_{j:\, \mathbf{W}_{c^*,j}z_j<0} |\mathbf{W}_{c^*,j}z_j|,
\end{align}
where $W_{c^*,j}$ denotes the $j^{\text{th}}$ component of $\mathbf{W}_{c^*}$, $C_p$ is the contribution to the max logit from components $z_j$ whose sign matches that of $W_{c^*,j}$, and $C_n$ is the contribution from sign-mismatched components. Since scaling multiplies every component of $\mathbf{z}$ by the same factor $\exp(r)$, it acts uniformly on this decomposition, giving $\ell'_{c^*} = \exp(r)\,\ell_{c^*} = \exp(r)(C_p-C_n)$. Under this decomposition, $\exp(r)C_p$ is the term that increases the max logit, whereas $\exp(r)C_n$ is the term that decreases it. We define the fraction of the terms constituting the logit that exert a negative influence as the \textbf{Anti-coherence}, denoted $A(\mathbf{z})$ and formulated as:
\begin{align}
A(\mathbf{z}) = \frac{C_n}{C_p+C_n} = \frac{\displaystyle\sum_{j:\, \mathbf{W}_{c^*,j}z_j<0} |\mathbf{W}_{c^*,j}z_j|}{\displaystyle\sum_{j:\, \mathbf{W}_{c^*,j}z_j>0} \mathbf{W}_{c^*,j}z_j + \sum_{j:\, \mathbf{W}_{c^*,j}z_j<0} |\mathbf{W}_{c^*,j}z_j|}.
\end{align}

\begin{table}[!t]
\centering
\caption{Mean ratio $\rho$ (ID/OOD) across ten backbones, without scaling (Raw Energy) and after applying our scaling $r_{\text{abs}}$.}
\label{tab:maxlogit_energy_ratio}
\resizebox{\textwidth}{!}{%
\begin{tabular}{lcccccccccc}
\toprule
 & ResNet-50 & ResNet-101 & Swin-B & SwinV2-B & ViT-B/16 & ConvNeXt-B & EfficientNet-B1 & MobileNetV2 & BEiTv2-B & DeiT3-B \\
\midrule
Raw (Energy)      & 0.976/0.923 & 0.981/0.934 & 0.971/0.928 & 0.970/0.921 & 0.962/0.906 & 0.943/0.879 & 0.947/0.881 & 0.970/0.922 & 0.970/0.908 & 0.780/0.318 \\
SCALE+Abs & 1.000/0.999 & 1.000/0.999 & 1.000/1.000 & 1.000/1.000 & 1.000/1.000 & 1.000/1.000 & 0.999/0.998 & 1.000/0.999 & 0.999/0.997 & 0.997/0.989 \\
\bottomrule
\end{tabular}%
}
\end{table}

\begin{figure}[t]
\centering
\includegraphics[width=\textwidth]{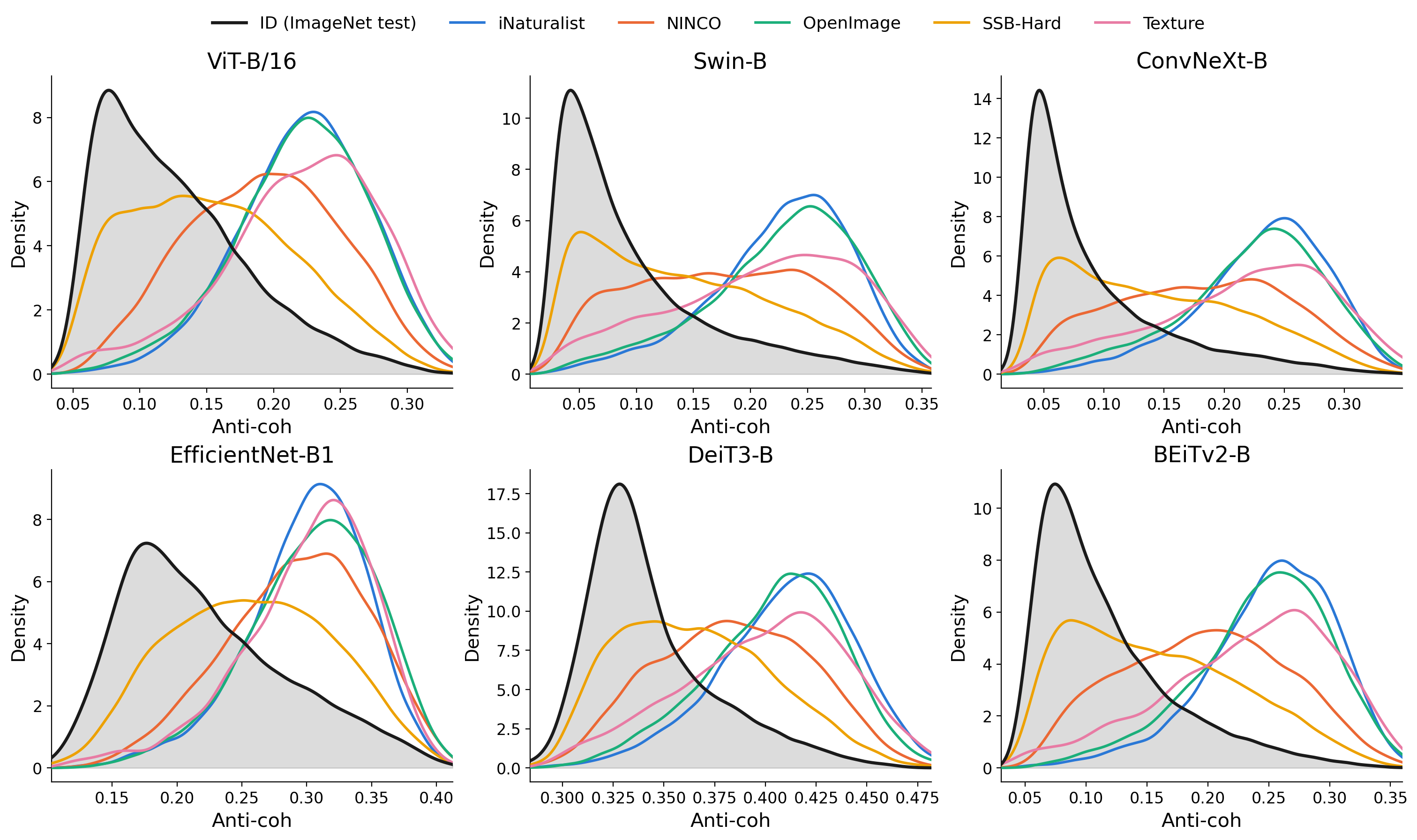}
\caption{Distribution of the Anti-coherence term $A(\mathbf{z})$ for ID and OOD samples across six backbones. OOD samples tend to exhibit higher $A(\mathbf{z})$ values than ID samples on all backbones. Additionally, SwinV2-B is omitted as it exhibits a similar pattern to Swin-B.}
\label{fig:anti_coh_main}
\end{figure}

Note that $\exp(r) > 0$ is order-preserving so that under the zero-bias assumption the identity of the predicted class $c^*$ remains invariant under scaling. Consequently, the sign of each $\mathbf{W}_{c^*,j}z_j$, and the partition into $C_p$ and $C_n$, is also unaffected by $r$, which makes $A(\mathbf{z})$ itself independent of $r$: it can be computed from either the pre-scaling or post-scaling activation without affecting its value. This is consistent with Lemma~\ref{lem:maxlogit_energy}, where scaling alters only how much the max logit contributes to the energy score rather than which logit attains the maximum or how it decomposes into $C_p$ and $C_n$.

To verify that $A(\mathbf{z})$ provides a signal meaningful for OOD detection, we visualize $A(\mathbf{z})$ on ID and OOD ImageNet data under the OpenOOD v1.5 protocol, with the results shown in Figure~\ref{fig:anti_coh_main}. The results show that the fraction of negatively contributing terms tends to be smaller for ID samples than for OOD samples on all backbones. This suggests that samples exhibiting such a larger fraction are more likely to be OOD and we therefore use it as a regularization term on the scaling factor, yielding
\begin{align}
r_{\text{anti}} = r_{\text{abs}} - \lambda A(\mathbf{z}),
\end{align}
where $\lambda \geq 0$ is a non-negative hyperparameter controlling the strength of the regularization. This regularization suppresses the scaling factor in proportion to the strength of the sign-mismatched contribution and thereby incorporates information from the classifier head directly into the scaling operation. Since $A(\mathbf{z})$ requires only the activation and classifier weight of a single test sample, our method remains practical and requires neither access to training data nor gradient computation, unlike existing methods that rely on such resources, including CADRef, which partitions by a similar sign criterion but requires class-conditional training averages to compute its error terms.

\section{Experiments}
\label{sec:exp}

\subsection{Experimental Setup}
\label{subsec:setup}
Following the OpenOOD v1.5 protocol, we evaluate on both a small-scale and a large-scale benchmark. For the small-scale benchmark, we adopt CIFAR-10 as the ID dataset with a ResNet-18 backbone, using the pretrained checkpoint released with OpenOOD v1.5. MNIST~\citep{deng2012mnist}, SVHN~\citep{netzer2011reading}, Texture~\citep{cimpoi2014describing}, and Places365~\citep{zhou2017places} serve as far-OOD datasets, and TIN~\citep{le2015tiny} and CIFAR-100 as near-OOD datasets; results for each individual OOD dataset are reported in Table~\ref{tab:cifar10_main_exp}.

For the large-scale benchmark, we adopt ImageNet as the ID dataset and evaluate across ten backbone architectures spanning convolutional and transformer families: ResNet-50, ResNet-101, Swin-B, SwinV2-B~\citep{liu2022swin}, ViT-B/16, ConvNeXt-B~\citep{liu2022convnet}, EfficientNet-B1~\citep{tan2019efficientnet}, MobileNetV2~\citep{sandler2018mobilenetv2}, BEiTv2-B~\citep{peng2022beit}, and DeiT3-B~\citep{touvron2022deit}. Pretrained weights are obtained from torchvision~\citep{torchvision2016} where available and from the timm library~\citep{rw2019timm} otherwise. iNaturalist~\citep{van2018inaturalist}, Texture, and OpenImage-O~\citep{wang2022vim} serve as far-OOD datasets and SSB-Hard~\citep{vaze2021open} and NINCO~\citep{bitterwolf2023or} as near-OOD datasets; Table~\ref{tab:imagenet_main_exp} reports results for each backbone averaged over the five OOD datasets.

We compare our method against activation-shaping methods and basic post-hoc baselines:
MSP \citep{hendrycks2017a}, MaxLogit \citep{hendrycks2022scaling}, Energy, ReAct, ASH-S, SCALE, VRA, OptFS, LTS \citep{djurisic2025logit}, LAPS, and AdaSCALE-A/L \citep{regmi2026adascale}. We exclude DAVIS and Catalyst, which require a GAP-based architecture, and BATS, which requires batch normalization, as these methods cannot be implemented in an architecture-agnostic manner across all backbones considered. We refer to our method, which scales the activation using $r_{\text{anti}}$, as \textbf{Anti-coh}. We report AUROC and FPR@95, the two metrics most commonly used to measure OOD detection performance, for all methods. Hyperparameter search spaces and other implementation details are provided in Appendix~\ref{appendix:implementation}, and the hyperparameter sensitivity analysis is provided in Appendix~\ref{appendix:hyp}.

\subsection{Main Results}
\label{subsec:main_results}

Table~\ref{tab:imagenet_main_exp} shows that our method achieves the lowest average FPR@95 across all ten backbones and the highest average AUROC on six of the ten. Our method also improves upon SCALE, on which it is built, across every backbone and shows especially large gains on backbones with negative penultimate activations. On ResNet and MobileNet, however, the improvement over SCALE is comparatively modest. Since these architectures always yield non-negative activations, the original scaling factor already captures sufficient information about ID/OOD separability, leaving less additional signal for the sign-based Anti-coherence term to contribute.

As for the four backbones where our method ranks second in AUROC, namely SwinV2-B, DeiT3-B, EfficientNet-B1, and BEiTv2-B, the gap to the best-performing method is only 0.07, 0.37, 1.10, and 0.04 points, respectively, with LAPS attaining the best AUROC on SwinV2-B, DeiT3-B, and BEiTv2-B and OptFS on EfficientNet-B1. Despite requiring no access to training data, unlike LAPS and OptFS, our method thus ties or nearly ties these training-data-dependent methods in AUROC while consistently improving upon them in FPR@95. 

\begin{table}[t]
\centering
\caption{OOD detection performance (AUROC/FPR@95, \%) averaged over five OOD datasets, across ten ImageNet backbones under the OpenOOD v1.5 protocol. Methods are grouped by resource requirement. The best value in each column is highlighted in \textbf{bold}.}
\label{tab:imagenet_main_exp}
\resizebox{\textwidth}{!}{%
\begin{tabular}{lccccccccccc}
\toprule
Method & RN-50 & Swin-B & ViT-B/16 & SwinV2-B & ConvNeXt-B & RN-101 & DeiT3-B & EfficientNet-B1 & MobileNetV2 & BEiTv2-B & Avg \\
\midrule
\multicolumn{12}{c}{\textit{Training-data-dependent}} \\
\midrule
ReAct       & 86.92/50.72 & 81.47/68.06 & 78.71/70.81 & 84.87/59.48 & 80.10/74.95 & 87.65/49.18 & 74.73/82.19 & 77.71/85.36 & 83.36/61.45 & 84.94/54.49 & 82.05/65.67 \\
VRA         & 87.69/46.53 & 71.80/83.00 & 41.40/90.44 & 77.49/83.27 & 69.10/80.89 & 88.13/44.23 & 41.36/94.86 & 78.33/74.65 & 84.01/56.44 & 38.95/96.07 & 67.83/75.04 \\
OptFS       & 87.64/47.38 & 83.12/63.55 & 82.25/66.98 & 84.13/62.06 & 85.50/59.41 & 88.00/47.25 & 83.31/61.89 & \textbf{83.62}/58.66 & 83.98/60.11 & 85.81/53.92 & 84.74/58.12 \\
LAPS        & 86.34/52.15 & 83.22/67.58 & 80.25/71.50 & \textbf{85.13}/61.99 & 84.76/68.22 & 87.37/49.96 & \textbf{83.81}/66.59 & 79.41/74.50 & 86.40/47.35 & \textbf{86.04}/54.77 & 84.27/61.46 \\
\midrule
\multicolumn{12}{c}{\textit{Gradient-dependent}} \\
\midrule
AdaSCALE-A  & 88.85/47.66 & 79.06/66.18 & 78.62/68.03 & 81.15/60.44 & 83.65/63.81 & 89.05/46.79 & 82.78/61.06 & 78.59/68.24 & 84.48/59.47 & 83.71/54.00 & 82.99/59.57 \\
AdaSCALE-L  & 89.81/42.79 & 78.51/69.53 & 75.87/74.40 & 80.14/62.76 & 82.61/73.38 & 90.03/42.49 & 82.23/65.51 & 75.23/79.84 & 85.70/53.80 & 82.21/57.63 & 82.23/62.21 \\
\midrule
\multicolumn{12}{c}{\textit{Training-free, gradient-free}} \\
\midrule
MSP         & 81.21/68.59 & 79.33/66.07 & 80.62/65.99 & 81.75/62.90 & 80.33/62.75 & 81.19/68.08 & 83.01/61.05 & 79.47/68.31 & 79.00/74.55 & 83.81/56.42 & 80.97/65.47 \\
MaxLogit    & 84.00/64.49 & 72.15/67.56 & 77.08/65.60 & 78.12/61.89 & 73.80/69.17 & 84.52/62.49 & 81.49/62.10 & 74.38/72.76 & 81.61/69.15 & 81.51/53.97 & 78.87/64.92 \\
Energy      & 83.71/64.91 & 64.20/76.41 & 72.05/70.83 & 72.82/68.84 & 57.63/87.01 & 84.31/63.05 & 77.71/75.02 & 66.60/87.28 & 81.33/68.80 & 78.72/59.03 & 73.91/72.12 \\
ASH-S       & 89.53/40.20 & 23.63/98.94 & 22.89/99.07 & 20.41/99.13 & 26.60/97.05 & 89.68/39.83 & 52.26/90.67 & 47.94/98.59 & 85.47/50.74 & 20.13/99.15 & 47.85/81.34 \\
SCALE       & 90.24/38.97 & 63.58/87.50 & 68.31/72.90 & 51.24/89.40 & 64.32/92.10 & 90.37/38.73 & 80.49/65.09 & 50.79/97.87 & 86.66/47.43 & 66.84/73.95 & 71.28/70.39 \\
LTS         & 90.12/39.38 & 76.71/66.93 & 79.35/63.43 & 79.66/60.71 & 80.37/62.02 & 90.34/38.91 & 82.04/60.75 & 64.32/86.78 & 86.64/47.72 & 78.23/57.83 & 80.78/58.45 \\
Anti-coh (Ours) & \textbf{90.65}/\textbf{38.65} & \textbf{83.29}/\textbf{61.39} & \textbf{82.92}/\textbf{63.18} & 85.06/\textbf{56.00} & \textbf{86.32}/\textbf{57.33} & \textbf{90.62}/\textbf{38.46} & 83.44/\textbf{57.21} & 82.52/\textbf{57.48} & \textbf{87.27}/\textbf{46.10} & 86.00/\textbf{51.82} & \textbf{85.81}/\textbf{52.76} \\
\bottomrule
\end{tabular}%
}
\end{table}

\begin{table}[t]
\centering
\caption{OOD detection performance (AUROC/FPR@95, \%) on CIFAR-10 (ResNet-18) under the OpenOOD v1.5 protocol. Methods are grouped by resource requirement. The best value in each column is highlighted in \textbf{bold}.}
\label{tab:cifar10_main_exp}
\resizebox{\textwidth}{!}{%
\begin{tabular}{lcccccccc}
\toprule
\multirow{2}{*}{Method} & \multicolumn{2}{c}{Near-OOD} & \multicolumn{4}{c}{Far-OOD} & \multirow{2}{*}{Avg} \\
\cmidrule(lr){2-3} \cmidrule(lr){4-7}
 & TIN & CIFAR-100 & MNIST & Places365 & SVHN & Texture & \\
\midrule
\multicolumn{8}{c}{\textit{Training-data-dependent}} \\
\midrule
ReAct       & 86.25/50.24 & 84.87/56.02 & 99.26/4.40 & 90.20/42.83 & 88.74/53.55 & 86.62/53.60 & 89.32/43.44 \\
VRA         & 86.31/50.16 & 84.96/55.90 & 99.25/4.44 & 90.21/42.89 & 88.82/53.34 & 86.76/53.46 & 89.39/43.36 \\
OptFS       & 85.63/54.59 & 84.38/59.60 & 98.96/5.90 & 89.90/46.30 & 87.95/59.09 & 85.70/58.28 & 88.75/47.29 \\
LAPS        & 86.91/49.42 & 85.68/54.68 & 99.29/4.21 & 89.34/44.62 & 90.01/49.48 & 88.72/49.36 & 89.99/41.96 \\
\midrule
\multicolumn{8}{c}{\textit{Gradient-dependent}} \\
\midrule
AdaSCALE-A  & 83.65/50.79 & 81.99/56.04 & 98.62/8.12 & 88.63/41.84 & 89.20/49.24 & 84.18/52.15 & 87.71/43.03 \\
AdaSCALE-L  & 77.94/53.45 & 75.86/58.37 & 96.49/16.89 & 87.41/\textbf{39.95} & 83.52/54.36 & 77.28/56.40 & 83.08/46.57 \\
\midrule
\multicolumn{8}{c}{\textit{Training-free, gradient-free}} \\
\midrule
MSP         & 86.73/58.84 & 86.13/61.86 & 97.51/16.08 & 87.89/56.38 & 90.50/58.41 & 88.30/58.76 & 89.51/51.72 \\
MaxLogit    & 86.36/48.52 & 84.95/53.75 & 99.30/3.52 & 88.11/44.60 & 90.92/45.21 & 87.67/48.85 & 89.55/40.74 \\
Energy      & 86.44/48.25 & 85.00/53.82 & \textbf{99.43}/\textbf{3.06} & 88.22/44.03 & 91.04/\textbf{44.37} & 87.74/48.78 & 89.64/\textbf{40.38} \\
ASH-S       & 81.52/51.39 & 79.26/57.06 & 98.50/9.25 & 87.05/44.10 & 83.20/56.87 & 81.87/51.58 & 85.23/45.04 \\
SCALE       & 82.58/50.77 & 80.41/56.73 & 98.87/6.28 & 87.79/42.23 & 84.89/55.86 & 82.71/52.93 & 86.21/44.13 \\
LTS         & 75.10/54.35 & 71.77/60.39 & 97.34/12.13 & 85.11/41.18 & 70.30/65.27 & 73.33/58.24 & 78.82/48.59 \\
Anti-coh (Ours) & \textbf{89.04}/\textbf{47.15} & \textbf{88.18}/\textbf{52.72} & 98.89/6.56 & \textbf{91.16}/41.45 & \textbf{91.34}/49.90 & \textbf{91.12}/\textbf{46.28} & \textbf{91.62}/40.68 \\
\bottomrule
\end{tabular}%
}
\vspace{-0.2cm}
\end{table}

On the small-scale benchmark (Table~\ref{tab:cifar10_main_exp}), our method outperforms SCALE on most individual datasets and achieves the highest average AUROC among baselines. As for the gain over SCALE, recall that ResNet-18 always yields non-negative activations and this was the setting where the gain was previously modest on the large-scale benchmark. On CIFAR-10, however, the pattern differs: our method achieves substantial gains over SCALE on most datasets and delivers overall stronger performance. Compared to methods that require additional resources such as training data or gradients, our method achieves the best overall average performance despite trailing several of them in FPR@95 on MNIST and Places365, most notably AdaSCALE-L on Places365 by 1.50 points. Interestingly, the plain Energy baseline achieves strong performance on MNIST and SVHN, yet fails to maintain such strong performance on the remaining datasets.

In summary, our method consistently improves upon SCALE across both ID settings and, within the training-free, gradient-free group, attains the best average AUROC and FPR@95 on every individual ImageNet backbone and the best average AUROC on CIFAR-10. Beyond this group, it further ties or outperforms methods that require additional resources such as training data or gradients on most backbones in both settings. We additionally evaluate on CIFAR-100 using the pretrained checkpoint provided by OpenOOD. Our method underperforms several baselines here, but most methods attain noticeably lower absolute performance on CIFAR-100 than on CIFAR-10 or ImageNet with ResNet backbones, which is somewhat unexpected given the smaller scale. We suspect this stems from the limited number of images per class, and discuss this along with per-dataset results for every backbone under the ImageNet ID setting further in Appendix~\ref{appendix:full_results}.


\subsection{Ablation Studies}
\label{subsec:ablation}
\textbf{Effect of Absolute Fix.} To verify that the Absolute Fix yields a genuine improvement, we report in Table~\ref{tab:abs_fix_ablation_avg} the performance of the original SCALE and the ReLU-based fixing method proposed in LTS, compared against our method, on backbones with negative penultimate activations. Our method not only improves performance over SCALE on every model but also achieves a consistent improvement over the ReLU-based method while always outperforming the plain Energy baseline. Notably, on BEiTv2-B the ReLU-based method actually underperforms Energy. This improvement remains substantial on EfficientNet-B1, despite lacking the symmetric heavy tail that originally motivated the Absolute Fix, and on BEiTv2-B, whose activation distribution resembles that of ViT. On DeiT3-B, however, our method leaves AUROC unchanged, which we attribute to fluctuations in the activation density at large absolute values rather than a clean monotonic decay, causing confusion when estimating the scaling factor. Detailed per-dataset performance is provided in Appendix~\ref{appendix:absolute_fix}.

\begin{table}[h]
\centering
\caption{Average AUROC/FPR@95 (\%) of the Absolute Fix against Energy, SCALE, and the ReLU-based correction, averaged over five OOD datasets, on backbones with negative penultimate activations under the ImageNet ID setting. $\Delta(\text{SCALE})$ and $\Delta(\text{ReLU})$ denote the change from SCALE and SCALE+ReLU to SCALE+Abs. The best value in each column is highlighted in \textbf{bold}.}
\label{tab:abs_fix_ablation_avg}
\resizebox{\textwidth}{!}{%
\begin{tabular}{lccccccc}
\toprule
Method & Swin-B & SwinV2-B & ViT-B/16 & ConvNeXt-B & DeiT3-B & EfficientNet-B1 & BEiTv2-B \\
\midrule
Energy (Baseline) & 64.20/76.41 & 72.82/68.84 & 72.05/70.83 & 57.63/87.01 & 77.71/75.02 & 66.60/87.28 & 78.72/59.03 \\
SCALE             & 63.58/87.50 & 51.24/89.40 & 68.31/72.90 & 64.32/92.10 & \textbf{80.49}/65.09 & 50.79/97.87 & 66.84/73.95 \\
SCALE+ReLU        & 76.06/70.90 & 76.72/66.88 & 78.73/63.94 & 79.62/65.80 & 80.19/66.26 & 67.30/77.54 & 75.38/62.96 \\
SCALE+Abs (Ours)  & \textbf{77.51}/\textbf{69.03} & \textbf{81.05}/\textbf{61.72} & \textbf{81.59}/\textbf{61.10} & \textbf{81.45}/\textbf{64.08} & \textbf{80.49}/\textbf{62.53} & \textbf{80.89}/\textbf{62.79} & \textbf{81.17}/\textbf{55.14} \\
\midrule
$\Delta$(SCALE) & +13.93/$-$18.47 & +29.81/$-$27.68 & +13.28/$-$11.80 & +17.13/$-$28.02 & +0.00/$-$2.56 & +30.10/$-$35.08 & +14.33/$-$18.81 \\
$\Delta$(ReLU)  & +1.45/$-$1.87 & +4.33/$-$5.16 & +2.86/$-$2.84 & +1.83/$-$1.72 & +0.30/$-$3.73 & +13.59/$-$14.75 & +5.79/$-$7.82 \\
\bottomrule
\end{tabular}%
}
\end{table}

\textbf{Effect of Anti-coherence.} Building on the Absolute Fix, we further isolate the effect of the Anti-coherence term $A(\mathbf{z})$ by comparing Anti-coh against SCALE+Abs, which already incorporates the gains from the Absolute Fix. The results show improvement on every metric except FPR@95 on ViT-B/16. In particular, the improvement is largest on backbones with negative penultimate activations, whereas, as noted in Section~\ref{subsec:main_results}, backbones with only non-negative activations show a comparatively modest gain. These results indicate that the Anti-coherence term provides a contribution beyond the Absolute Fix alone.

\begin{table}[h]
\centering
\caption{Average AUROC/FPR@95 (\%) of Anti-coh (Ours) against SCALE+Abs, averaged over five OOD datasets across ten ImageNet backbones. $\Delta$ denotes the change from SCALE+Abs to Anti-coh. The best value in each column is highlighted in \textbf{bold}.}
\label{tab:scale_absfix_anticoh_progression}
\resizebox{\textwidth}{!}{%
\begin{tabular}{lcccccccccc}
\toprule
Method & RN-50 & RN-101 & MobileNetV2 & Swin-B & SwinV2-B & ViT-B/16 & ConvNeXt-B & DeiT3-B & EfficientNet-B1 & BEiTv2-B \\
\midrule
SCALE+Abs   & 90.24/38.97 & 90.37/38.73 & 86.66/47.43 & 77.51/69.03 & 81.05/61.72 & 81.59/\textbf{61.10} & 81.45/64.08 & 80.49/62.53 & 80.89/62.79 & 81.17/55.14 \\
Anti-coh (Ours)  & \textbf{90.65}/\textbf{38.65} & \textbf{90.62}/\textbf{38.46} & \textbf{87.27}/\textbf{46.10} & \textbf{83.29}/\textbf{61.39} & \textbf{85.06}/\textbf{56.00} & \textbf{82.92}/63.18 & \textbf{86.32}/\textbf{57.33} & \textbf{83.44}/\textbf{57.21} & \textbf{82.52}/\textbf{57.48} & \textbf{86.00}/\textbf{51.82} \\
\midrule
$\Delta$    & +0.41/$-$0.32 & +0.25/$-$0.27 & +0.61/$-$1.33 & +5.78/$-$7.64 & +4.01/$-$5.72 & +1.33/+2.08 & +4.87/$-$6.75 & +2.95/$-$5.32 & +1.63/$-$5.31 & +4.83/$-$3.32 \\
\bottomrule
\end{tabular}%
}
\end{table}

\section{Conclusion}
\label{sec:conclusion}

Towards addressing OOD detection, we show that SCALE fails to compute a well-defined scaling factor on backbones with negative penultimate activations, since positive and negative contributions cancel within the underlying signed sum. We propose the Absolute Fix, which computes the scaling factor over absolute-valued activations to preserve this information, and Anti-coherence, which decomposes each sample's contribution to the maximum logit into sign-aligned and sign-mismatched terms and regularizes the scaling factor by the fraction that drives the logit downward. Under the OpenOOD v1.5 protocol, our method requires no training data or gradient computation yet achieves the best average AUROC and FPR@95 among training-free methods on every individual ImageNet backbone and the best average AUROC on CIFAR-10, tying or outperforming methods relying on such resources on most backbones in both settings. These gains are especially large on backbones with negative activations though a gap with representation-based methods remains (see Appendix~\ref{appendix:Rep_method}), and closing it without relying on training data is an important direction for future work.

\subsection*{AI use statement}

In this work, we used generative AI tools to assist with implementing methods and with retrieval and discovery of related work. We have not used generative AI tools to propose or refine hypotheses; the core ideas of this work originate solely from the authors. We also used generative AI tools to verify the correctness of the mathematical proof, checking each derivation step against the stated assumptions. Additionally, we used generative AI tools for editing the paper to improve readability and clarity of writing. We have reviewed all AI-assisted work: LLM-generated code was verified and tested for correctness by the authors, the proof verification was independently confirmed by the authors, and all AI-edited text was manually checked against the intended scientific content. We take responsibility for the final content of this work, including text, claims, or artifacts produced with the aid of generative AI.

\subsection*{Ethics statement}

This work proposes an algorithmic improvement to OOD detection and is evaluated exclusively on established public benchmark datasets under the OpenOOD v1.5 protocol. It does not involve human subjects, new data collection, or the release of any dataset. Although the proposed method is intended to improve the reliability of deployed machine learning systems, no OOD detection method offers a strict guarantee of coverage and should not be treated as a sole safeguard in high-stakes deployment settings. We do not foresee privacy, discrimination, or conflict-of-interest concerns arising from this research.

\subsection*{Reproducibility statement}
Our method builds on publicly available ImageNet-1K pretrained checkpoints from torchvision~\citep{torchvision2016} and the timm library~\citep{rw2019timm}, and is evaluated on the publicly available OpenOOD v1.5 benchmark~\citep{zhang2024openood}, which specifies the ID/OOD dataset splits and evaluation protocol used throughout the paper. The complete proof of Lemma~\ref{lem:maxlogit_energy} is provided in Appendix~\ref{appendix:proof}, along with the assumptions under which it holds. Hyperparameter search spaces, checkpoint details, and other implementation details are provided in Appendix~\ref{appendix:implementation} and source code is included in the supplementary material.



\bibliography{iclr2027_conference}

@inproceedings{isaku2025out,
  title={Out of Distribution Detection in Self-adaptive Robots with AI-powered Digital Twins},
  author={Isaku, Erblin and Sartaj, Hassan and Ali, Shaukat and Sanguino, Beatriz and Wang, Tongtong and Li, Guoyuan and Zhang, Houxiang and Peyrucain, Thomas},
  booktitle={2025 40th IEEE/ACM International Conference on Automated Software Engineering (ASE)},
  pages={3403--3414},
  year={2025},
  organization={IEEE}
}

@inproceedings{corsini2023existing,
  title={Are existing out-of-distribution techniques suitable for network intrusion detection?},
  author={Corsini, Andrea and Yang, Shanchieh Jay},
  booktitle={2023 IEEE Conference on Communications and Network Security (CNS)},
  pages={1--9},
  year={2023},
  organization={IEEE}
}

@article{yang2024generalized,
  title={Generalized out-of-distribution detection: A survey},
  author={Yang, Jingkang and Zhou, Kaiyang and Li, Yixuan and Liu, Ziwei},
  journal={International Journal of Computer Vision},
  volume={132},
  number={12},
  pages={5635--5662},
  year={2024},
  publisher={Springer}
}

@article{hong2024out,
  title={Out-of-distribution detection in medical image analysis: A survey},
  author={Hong, Zesheng and Yue, Yubiao and Chen, Yubin and Cong, Lele and Lin, Huanjie and Luo, Yuanmei and Wang, Mini Han and Wang, Weidong and Xu, Jialong and Yang, Xiaoqi and others},
  journal={arXiv preprint arXiv:2404.18279},
  year={2024}
}

@inproceedings{
ming2023how,
title={How to Exploit Hyperspherical Embeddings for Out-of-Distribution Detection?},
author={Yifei Ming and Yiyou Sun and Ousmane Dia and Yixuan Li},
booktitle={The Eleventh International Conference on Learning Representations },
year={2023},
url={https://openreview.net/forum?id=aEFaE0W5pAd}
}

@inproceedings{
du2022vos,
title={VOS: Learning What You Don't Know by Virtual Outlier Synthesis },
author={Xuefeng Du and Zhaoning Wang and Mu Cai and Sharon Li},
booktitle={International Conference on Learning Representations},
year={2022},
url={https://openreview.net/forum?id=TW7d65uYu5M}
}

@article{zhu2023diversified,
  title={Diversified outlier exposure for out-of-distribution detection via informative extrapolation},
  author={Zhu, Jianing and Geng, Yu and Yao, Jiangchao and Liu, Tongliang and Niu, Gang and Sugiyama, Masashi and Han, Bo},
  journal={Advances in neural information processing systems},
  volume={36},
  pages={22702--22734},
  year={2023}
}

@inproceedings{
tao2023nonparametric,
title={Non-parametric Outlier Synthesis},
author={Leitian Tao and Xuefeng Du and Jerry Zhu and Yixuan Li},
booktitle={The Eleventh International Conference on Learning Representations },
year={2023},
url={https://openreview.net/forum?id=JHklpEZqduQ}
}

@article{liu2020energy,
  title={Energy-based out-of-distribution detection},
  author={Liu, Weitang and Wang, Xiaoyun and Owens, John and Li, Yixuan},
  journal={Advances in neural information processing systems},
  volume={33},
  pages={21464--21475},
  year={2020}
}

@inproceedings{
xu2024scaling,
title={Scaling for Training Time and Post-hoc Out-of-distribution Detection Enhancement},
author={Kai Xu and Rongyu Chen and Gianni Franchi and Angela Yao},
booktitle={The Twelfth International Conference on Learning Representations},
year={2024},
url={https://openreview.net/forum?id=RDSTjtnqCg}
}

@inproceedings{
djurisic2023extremely,
title={Extremely Simple Activation Shaping for Out-of-Distribution Detection},
author={Andrija Djurisic and Nebojsa Bozanic and Arjun Ashok and Rosanne Liu},
booktitle={The Eleventh International Conference on Learning Representations },
year={2023},
url={https://openreview.net/forum?id=ndYXTEL6cZz}
}

@article{sun2021react,
  title={React: Out-of-distribution detection with rectified activations},
  author={Sun, Yiyou and Guo, Chuan and Li, Yixuan},
  journal={Advances in neural information processing systems},
  volume={34},
  pages={144--157},
  year={2021}
}

@inproceedings{
regmi2026adascale,
title={Ada{SCALE}: Adaptive Scaling for {OOD} Detection},
author={Sudarshan Regmi},
booktitle={Forty-third International Conference on Machine Learning},
year={2026},
url={https://openreview.net/forum?id=iXPK0rsCE9}
}

@inproceedings{liu2021swin,
  title={Swin transformer: Hierarchical vision transformer using shifted windows},
  author={Liu, Ze and Lin, Yutong and Cao, Yue and Hu, Han and Wei, Yixuan and Zhang, Zheng and Lin, Stephen and Guo, Baining},
  booktitle={Proceedings of the IEEE/CVF international conference on computer vision},
  pages={10012--10022},
  year={2021}
}

@inproceedings{he2016deep,
  title={Deep residual learning for image recognition},
  author={He, Kaiming and Zhang, Xiangyu and Ren, Shaoqing and Sun, Jian},
  booktitle={Proceedings of the IEEE conference on computer vision and pattern recognition},
  pages={770--778},
  year={2016}
}

@inproceedings{
dosovitskiy2021an,
title={An Image is Worth 16x16 Words: Transformers for Image Recognition at Scale},
author={Alexey Dosovitskiy and Lucas Beyer and Alexander Kolesnikov and Dirk Weissenborn and Xiaohua Zhai and Thomas Unterthiner and Mostafa Dehghani and Matthias Minderer and Georg Heigold and Sylvain Gelly and Jakob Uszkoreit and Neil Houlsby},
booktitle={International Conference on Learning Representations},
year={2021},
url={https://openreview.net/forum?id=YicbFdNTTy}
}

@article{djurisic2025logit,
  title={Logit scaling for out-of-distribution detection},
  author={Djurisic, Andrija and Liu, Rosanne and Nikolic, Mladen},
  journal={Machine Vision and Applications},
  volume={36},
  number={5},
  pages={108},
  year={2025},
  publisher={Springer}
}

@inproceedings{NEURIPS2025_ee4432ad,
 author = {Wagner, Felix and Saha, Pramit and Anthony, Harry and Noble, Alison and Kamnitsas, Konstantinos},
 booktitle = {Advances in Neural Information Processing Systems},
 doi = {10.52202/085713-5433},
 editor = {D. Belgrave and C. Zhang and H. Lin and R. Pascanu and P. Koniusz and M. Ghassemi and N. Chen},
 pages = {162740--162769},
 publisher = {Curran Associates, Inc.},
 title = {DIsoN: Decentralized Isolation Networks for Out-of-Distribution Detection in Medical Imaging},
 url = {https://proceedings.neurips.cc/paper_files/paper/2025/file/ee4432ad0aa3ab4c6c3dee47b2d541b5-Paper-Conference.pdf},
 volume = {38, Main Conference},
 year = {2025}
}

@article{zhang2024openood,
  title={OpenOOD v1. 5: Enhanced Benchmark for Out-of-Distribution Detection},
  author={Zhang, Jingyang and Yang, Jingkang and Wang, Pengyun and Wang, Haoqi and Lin, Yueqian and Zhang, Haoran and Sun, Yiyou and Du, Xuefeng and Li, Yixuan and Liu, Ziwei and others},
  journal={Journal of Data-centric Machine Learning Research},
  year={2024},
  publisher={DMLR}
}

@inproceedings{deng2009imagenet,
  title={Imagenet: A large-scale hierarchical image database},
  author={Deng, Jia and Dong, Wei and Socher, Richard and Li, Li-Jia and Li, Kai and Fei-Fei, Li},
  booktitle={2009 IEEE conference on computer vision and pattern recognition},
  pages={248--255},
  year={2009},
  organization={Ieee}
}

@article{krizhevsky2009learning,
  title={Learning multiple layers of features from tiny images},
  author={Krizhevsky, Alex and Hinton, Geoffrey and others},
  year={2009},
  publisher={Toronto, ON, Canada}
}

@inproceedings{
hendrycks2017a,
title={A Baseline for Detecting Misclassified and Out-of-Distribution Examples in Neural Networks},
author={Dan Hendrycks and Kevin Gimpel},
booktitle={International Conference on Learning Representations},
year={2017},
url={https://openreview.net/forum?id=Hkg4TI9xl}
}

@inproceedings{hendrycks2022scaling,
  title={Scaling Out-of-Distribution Detection for Real-World Settings},
  author={Hendrycks, Dan and Basart, Steven and Mazeika, Mantas and Zou, Andy and Kwon, Joseph and Mostajabi, Mohammadreza and Steinhardt, Jacob and Song, Dawn},
  booktitle={International Conference on Machine Learning},
  pages={8759--8773},
  year={2022},
  organization={PMLR}
}

@article{ren2021simple,
  title={A simple fix to mahalanobis distance for improving near-ood detection},
  author={Ren, Jie and Fort, Stanislav and Liu, Jeremiah and Roy, Abhijit Guha and Padhy, Shreyas and Lakshminarayanan, Balaji},
  journal={arXiv preprint arXiv:2106.09022},
  year={2021}
}

@article{lee2018simple,
  title={A simple unified framework for detecting out-of-distribution samples and adversarial attacks},
  author={Lee, Kimin and Lee, Kibok and Lee, Honglak and Shin, Jinwoo},
  journal={Advances in neural information processing systems},
  volume={31},
  year={2018}
}

@inproceedings{muller2025mahalanobis++,
  title={Mahalanobis++: Improving OOD Detection via Feature Normalization},
  author={M{\"u}ller, Maximilian and Hein, Matthias},
  booktitle={International Conference on Machine Learning},
  pages={45151--45184},
  year={2025},
  organization={PMLR}
}

@inproceedings{park2023nearest,
  title={Nearest neighbor guidance for out-of-distribution detection},
  author={Park, Jaewoo and Jung, Yoon Gyo and Teoh, Andrew Beng Jin},
  booktitle={Proceedings of the IEEE/CVF international conference on computer vision},
  pages={1686--1695},
  year={2023}
}

@inproceedings{sun2022out,
  title={Out-of-distribution detection with deep nearest neighbors},
  author={Sun, Yiyou and Ming, Yifei and Zhu, Xiaojin and Li, Yixuan},
  booktitle={International conference on machine learning},
  pages={20827--20840},
  year={2022},
  organization={PMLR}
}

@inproceedings{liu2024fast,
  title={Fast Decision Boundary based Out-of-Distribution Detector},
  author={Liu, Litian and Qin, Yao},
  booktitle={International Conference on Machine Learning},
  pages={31728--31746},
  year={2024},
  organization={PMLR}
}

@inproceedings{liu2025detecting,
  title={Detecting out-of-distribution through the lens of neural collapse},
  author={Liu, Litian and Qin, Yao},
  booktitle={Proceedings of the Computer Vision and Pattern Recognition Conference},
  pages={15424--15433},
  year={2025}
}

@article{papyan2020prevalence,
  title={Prevalence of neural collapse during the terminal phase of deep learning training},
  author={Papyan, Vardan and Han, XY and Donoho, David L},
  journal={Proceedings of the National Academy of Sciences},
  volume={117},
  number={40},
  pages={24652--24663},
  year={2020},
  publisher={National Academy of Sciences}
}

@article{xu2023vra,
  title={VRA: Variational rectified activation for out-of-distribution detection},
  author={Xu, Mingyu and Lian, Zheng and Liu, Bin and Tao, Jianhua},
  journal={Advances in Neural Information Processing Systems},
  volume={36},
  pages={28941--28959},
  year={2023}
}

@inproceedings{sun2022dice,
  title={Dice: Leveraging sparsification for out-of-distribution detection},
  author={Sun, Yiyou and Li, Yixuan},
  booktitle={European conference on computer vision},
  pages={691--708},
  year={2022},
  organization={Springer}
}

@inproceedings{
zhao2024towards,
title={Towards Optimal Feature-Shaping Methods for Out-of-Distribution Detection},
author={Qinyu Zhao and Ming Xu and Kartik Gupta and Akshay Asthana and Liang Zheng and Stephen Gould},
booktitle={The Twelfth International Conference on Learning Representations},
year={2024},
url={https://openreview.net/forum?id=dm8e7gsH0d}
}

@InProceedings{Hassan_2026_CVPR,
    author    = {Hassan, Abid and Ngo, Tuan and Shafiq, Saad and Medvidovic, Nenad},
    title     = {Catalyst: Out-of-Distribution Detection via Elastic Scaling},
    booktitle = {Proceedings of the IEEE/CVF Conference on Computer Vision and Pattern Recognition (CVPR) Findings},
    month     = {June},
    year      = {2026},
    pages     = {1618-1628}
}

@article{hassan2026davis,
  title={DAVIS: OOD Detection via Dominant Activations and Variance for Increased Separation},
  author={Hassan, Abid and Ngo, Tuan and Shafiq, Saad and Medvidovic, Nenad},
  journal={arXiv preprint arXiv:2601.22703},
  year={2026}
}

@inproceedings{he2024exploring,
  title={Exploring channel-aware typical features for out-of-distribution detection},
  author={He, Rundong and Yuan, Yue and Han, Zhongyi and Wang, Fan and Su, Wan and Yin, Yilong and Liu, Tongliang and Gong, Yongshun},
  booktitle={Proceedings of the AAAI conference on artificial intelligence},
  volume={38},
  number={11},
  pages={12402--12410},
  year={2024}
}

@article{zhu2022boosting,
  title={Boosting out-of-distribution detection with typical features},
  author={Zhu, Yao and Chen, YueFeng and Xie, Chuanlong and Li, Xiaodan and Zhang, Rong and Xue, Hui and Tian, Xiang and Chen, Yaowu and others},
  journal={Advances in Neural Information Processing Systems},
  volume={35},
  pages={20758--20769},
  year={2022}
}

@inproceedings{wang2022vim,
  title={Vim: Out-of-distribution with virtual-logit matching},
  author={Wang, Haoqi and Li, Zhizhong and Feng, Litong and Zhang, Wayne},
  booktitle={Proceedings of the IEEE/CVF conference on computer vision and pattern recognition},
  pages={4921--4930},
  year={2022}
}

@InProceedings{Liu_2023_CVPR,
    author    = {Liu, Xixi and Lochman, Yaroslava and Zach, Christopher},
    title     = {GEN: Pushing the Limits of Softmax-Based Out-of-Distribution Detection},
    booktitle = {Proceedings of the IEEE/CVF Conference on Computer Vision and Pattern Recognition (CVPR)},
    month     = {June},
    year      = {2023},
    pages     = {23946-23955}
}

@inproceedings{ling2025cadref,
  title={Cadref: Robust out-of-distribution detection via class-aware decoupled relative feature leveraging},
  author={Ling, Zhiwei and Chang, Yachen and Zhao, Hailiang and Zhao, Xinkui and Chow, Kingsum and Deng, Shuiguang},
  booktitle={2025 IEEE/CVF Conference on Computer Vision and Pattern Recognition (CVPR)},
  pages={4968--4977},
  year={2025},
  organization={IEEE}
}

@article{deng2012mnist,
  title={The mnist database of handwritten digit images for machine learning research [best of the web]},
  author={Deng, Li},
  journal={IEEE signal processing magazine},
  volume={29},
  number={6},
  pages={141--142},
  year={2012},
  publisher={IEEE}
}

@inproceedings{netzer2011reading,
  title={Reading digits in natural images with unsupervised feature learning},
  author={Netzer, Yuval and Wang, Tao and Coates, Adam and Bissacco, Alessandro and Wu, Baolin and Ng, Andrew Y and others},
  booktitle={NIPS workshop on deep learning and unsupervised feature learning},
  volume={2011},
  number={2},
  pages={4},
  year={2011},
  organization={Granada}
}

@inproceedings{cimpoi2014describing,
  title={Describing textures in the wild},
  author={Cimpoi, Mircea and Maji, Subhransu and Kokkinos, Iasonas and Mohamed, Sammy and Vedaldi, Andrea},
  booktitle={Proceedings of the IEEE conference on computer vision and pattern recognition},
  pages={3606--3613},
  year={2014}
}

@article{zhou2017places,
  title={Places: A 10 million image database for scene recognition},
  author={Zhou, Bolei and Lapedriza, Agata and Khosla, Aditya and Oliva, Aude and Torralba, Antonio},
  journal={IEEE transactions on pattern analysis and machine intelligence},
  volume={40},
  number={6},
  pages={1452--1464},
  year={2017},
  publisher={IEEE}
}

@article{le2015tiny,
  title={Tiny imagenet visual recognition challenge},
  author={Le, Yann and Yang, Xuan and others},
  journal={CS 231N},
  volume={7},
  number={7},
  pages={3},
  year={2015}
}

@inproceedings{van2018inaturalist,
  title={The inaturalist species classification and detection dataset},
  author={Van Horn, Grant and Mac Aodha, Oisin and Song, Yang and Cui, Yin and Sun, Chen and Shepard, Alex and Adam, Hartwig and Perona, Pietro and Belongie, Serge},
  booktitle={Proceedings of the IEEE conference on computer vision and pattern recognition},
  pages={8769--8778},
  year={2018}
}

@inproceedings{vaze2021open,
  title={Open-set recognition: A good closed-set classifier is all you need},
  author={Vaze, Sagar and Han, Kai and Vedaldi, Andrea and Zisserman, Andrew},
  booktitle={International conference on learning representations},
  year={2021}
}

@inproceedings{bitterwolf2023or,
  title={In or Out? Fixing ImageNet Out-of-Distribution Detection Evaluation},
  author={Bitterwolf, Julian and M{\"u}ller, Maximilian and Hein, Matthias},
  booktitle={International Conference on Machine Learning},
  pages={2471--2506},
  year={2023},
  organization={PMLR}
}

@software{torchvision2016,
    title        = {TorchVision: PyTorch's Computer Vision library},
    author       = {TorchVision maintainers and contributors},
    year         = 2016,
    journal      = {GitHub repository},
    publisher    = {GitHub},
    howpublished = {\url{https://github.com/pytorch/vision}}
}

@misc{rw2019timm,
  author = {Ross Wightman},
  title = {PyTorch Image Models},
  year = {2019},
  publisher = {GitHub},
  journal = {GitHub repository},
  doi = {10.5281/zenodo.4414861},
  howpublished = {\url{https://github.com/rwightman/pytorch-image-models}}
}

@inproceedings{liu2022swin,
  title={Swin transformer v2: Scaling up capacity and resolution},
  author={Liu, Ze and Hu, Han and Lin, Yutong and Yao, Zhuliang and Xie, Zhenda and Wei, Yixuan and Ning, Jia and Cao, Yue and Zhang, Zheng and Dong, Li and others},
  booktitle={2022 IEEE/CVF Conference on Computer Vision and Pattern Recognition (CVPR)},
  pages={11999--12009},
  year={2022},
  organization={IEEE}
}

@inproceedings{liu2022convnet,
  title={A convnet for the 2020s},
  author={Liu, Zhuang and Mao, Hanzi and Wu, Chao-Yuan and Feichtenhofer, Christoph and Darrell, Trevor and Xie, Saining},
  booktitle={2022 IEEE/CVF conference on computer vision and pattern recognition (CVPR)},
  pages={11966--11976},
  year={2022},
  organization={IEEE}
}

@inproceedings{tan2019efficientnet,
  title={Efficientnet: Rethinking model scaling for convolutional neural networks},
  author={Tan, Mingxing and Le, Quoc},
  booktitle={International conference on machine learning},
  pages={6105--6114},
  year={2019},
  organization={PmLR}
}

@article{peng2022beit,
  title={Beit v2: Masked image modeling with vector-quantized visual tokenizers},
  author={Peng, Zhiliang and Dong, Li and Bao, Hangbo and Ye, Qixiang and Wei, Furu},
  journal={arXiv preprint arXiv:2208.06366},
  year={2022}
}

@inproceedings{touvron2022deit,
  title={Deit iii: Revenge of the vit},
  author={Touvron, Hugo and Cord, Matthieu and J{\'e}gou, Herv{\'e}},
  booktitle={European conference on computer vision},
  pages={516--533},
  year={2022},
  organization={Springer}
}

@inproceedings{sandler2018mobilenetv2,
  title={Mobilenetv2: Inverted residuals and linear bottlenecks},
  author={Sandler, Mark and Howard, Andrew and Zhu, Menglong and Zhmoginov, Andrey and Chen, Liang-Chieh},
  booktitle={2018 IEEE/CVF conference on computer vision and pattern recognition},
  pages={4510--4520},
  year={2018},
  organization={Ieee}
}

@article{kim2026mahavar,
  title={MahaVar: OOD Detection via Class-wise Mahalanobis Distance Variance under Neural Collapse},
  author={Kim, Donghwan and Yoon, Hyunsoo},
  journal={arXiv preprint arXiv:2605.14413},
  year={2026}
}

@article{paszke2019pytorch,
  title={Pytorch: An imperative style, high-performance deep learning library},
  author={Paszke, Adam and Gross, Sam and Massa, Francisco and Lerer, Adam and Bradbury, James and Chanan, Gregory and Killeen, Trevor and Lin, Zeming and Gimelshein, Natalia and Antiga, Luca and others},
  journal={Advances in neural information processing systems},
  volume={32},
  year={2019}
}
\bibliographystyle{iclr2027_conference}

\newpage
\appendix

\section{Proof of Lemma \ref{lem:maxlogit_energy}}
\label{appendix:proof}
\renewcommand{\thesection}{4}
\setcounter{theorem}{0}

\begin{lemma}
Let $\mathbf{W} \in \mathbb{R}^{C \times d}$ be the classifier head's weight, and let $\ell_k := \mathbf{W}_k \cdot \mathbf{z}$ with $c^* := \arg\max_k \ell_k$ the unique maximizer, $\ell_{c^*} > 0$. For simplicity, we assume the bias is zero. The scaled logit and energy score are
\begin{align}
\boldsymbol{\ell}' = \mathbf{W}\big(\exp(r)\cdot\mathbf{z}\big), \qquad S(\mathbf{x}) = T\cdot\log\sum_{k=1}^{C}\exp(\ell'_k/T),
\end{align}
which decomposes as
\begin{align}
S(\mathbf{x};r) = \ell'_{c^*}(r) + \xi(r), \qquad \xi(r) := T\log\Big(1+\sum_{k\neq c^*}\exp\big(\exp(r)(\ell_k-\ell_{c^*})/T\big)\Big) \ge 0,
\end{align}
where $\xi(r)$ is the contribution of the non-maximal logits. Let $\rho(r) := \ell'_{c^*}(r)/S(\mathbf{x};r)$ denote the fraction of the energy score attributable to the max logit. Then $\rho(r)$ is non-decreasing in $r$ with $\rho(r)\to1$ as $r\to\infty$.
\end{lemma}

\renewcommand{\thesection}{\Alph{section}}

\begin{proof}
By the shift-invariance of log-sum-exp,
\begin{align}
S(\mathbf{x};r) &= T\log\sum_{k=1}^{C}\exp(\ell'_k(r)/T) \\
&= T\log\left[\exp\!\left(\frac{\ell'_{c^*}(r)}{T}\right)\sum_{k=1}^{C}\exp\!\left(\frac{\ell'_k(r)-\ell'_{c^*}(r)}{T}\right)\right] \\
&= \ell'_{c^*}(r) + T\log\sum_{k=1}^{C}\exp\!\left(\frac{\ell'_k(r)-\ell'_{c^*}(r)}{T}\right) \\
&= \ell'_{c^*}(r) + T\log\Big(1+\sum_{k\neq c^*}\exp\big((\ell'_k(r)-\ell'_{c^*}(r))/T\big)\Big) =: \ell'_{c^*}(r) + \xi(r),
\end{align}
where the last line separates out the $k=c^*$ term, for which $\exp(0)=1$. Since $\ell'_k(r)=\exp(r)\ell_k$, we have $\ell'_k(r)-\ell'_{c^*}(r)=\exp(r)\delta_k$, where $\delta_k := \ell_k-\ell_{c^*} \le 0$ for all $k$ by the definition of $c^*$. Thus
\begin{align}
\xi(r) = T\log\Big(1+\sum_{k\neq c^*}\exp(\exp(r)\delta_k/T)\Big).
\end{align}
Let $g(r) := 1+\sum_{k\neq c^*}\exp(\exp(r)\delta_k/T)$, so that $\xi(r) = T\log g(r)$. Differentiating $g(r)$ with respect to $r$,
\begin{align}
g'(r) = \sum_{k\neq c^*}\frac{\delta_k\exp(r)}{T}\exp(\exp(r)\delta_k/T),
\end{align}
so that
\begin{align}
\xi'(r) = T\cdot\frac{g'(r)}{g(r)} = \frac{\exp(r)\sum_{k\neq c^*}\delta_k\exp(\exp(r)\delta_k/T)}{g(r)}.
\end{align}
The denominator $g(r)$ is strictly positive. In the numerator, $\exp(r)>0$ and each term $\delta_k\exp(\exp(r)\delta_k/T)\le0$ since $\delta_k\le0$. Hence $\xi'(r)\le0$, so $\xi(r)$ is non-increasing in $r$. Assuming $c^*$ is the unique maximizer, $\delta_k<0$ for all $k\neq c^*$, so $\exp(r)\delta_k\to-\infty$ as $r\to\infty$. Thus $\exp(\exp(r)\delta_k/T)\to0$ for every $k\neq c^*$, and $g(r)\to1$, giving
\begin{align}
\xi(r) \xrightarrow{r\to\infty} T\log(1) = 0.
\end{align}
Since $\ell_{c^*}>0$, $\ell'_{c^*}(r)=\exp(r)\ell_{c^*}$ is strictly increasing in $r$ and diverges as $r\to\infty$. As $\xi(r) \ge 0$ is non-increasing and bounded while $\ell'_{c^*}(r) > 0$ is strictly increasing and unbounded, the ratio $\xi(r)/\ell'_{c^*}(r)$ is non-increasing in $r$ and tends to $0$ as $r\to\infty$. Writing
\begin{align}
\rho(r) = \frac{\ell'_{c^*}(r)}{\ell'_{c^*}(r)+\xi(r)} = \Big(1+\frac{\xi(r)}{\ell'_{c^*}(r)}\Big)^{-1},
\end{align}
it follows that $\rho(r)$ is non-decreasing in $r$ and $\rho(r)\to1$ as $r\to\infty$.
\end{proof}

\section{Additional Ablations on the Absolute Fix}
\label{appendix:absolute_fix}

Figure~\ref{fig:qqp_abs_before_after} extends the analysis of Figure~\ref{fig:raw_activation_main} by visualizing how $r$ differs before and after applying the Absolute Fix across six backbones with negative penultimate activations. Consistent with the main text, the Absolute Fix shifts $r$ so that ID samples exhibit a higher value than OOD samples on all backbones except DeiT3-B, whereas the original $r$ frequently fails to separate ID from OOD or even reverses this ordering. This further confirms that the Absolute Fix restores the intended behavior of the scaling factor across a broader range of architectures with negative penultimate activations. We additionally examined SwinV2-B and found that, although the Absolute Fix separates ID from OOD even more clearly than on Swin-B, the qualitative pattern does not differ from Swin-B, and we therefore omit it from Figure \ref{fig:qqp_abs_before_after} for brevity.

\begin{figure}[!h]
\centering
\includegraphics[width=0.95\textwidth]{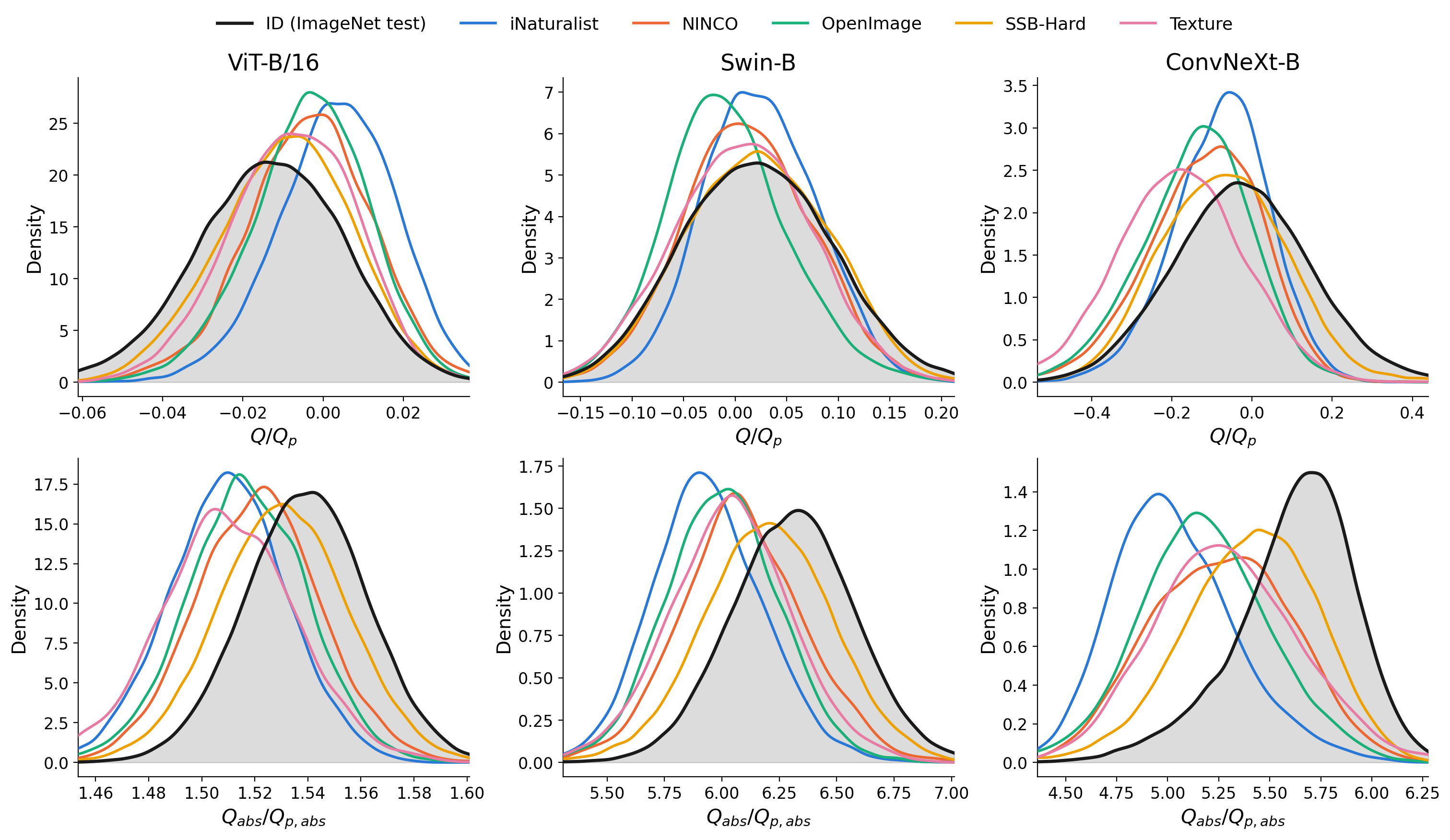}
\\[0.5em]
\includegraphics[width=0.95\textwidth]{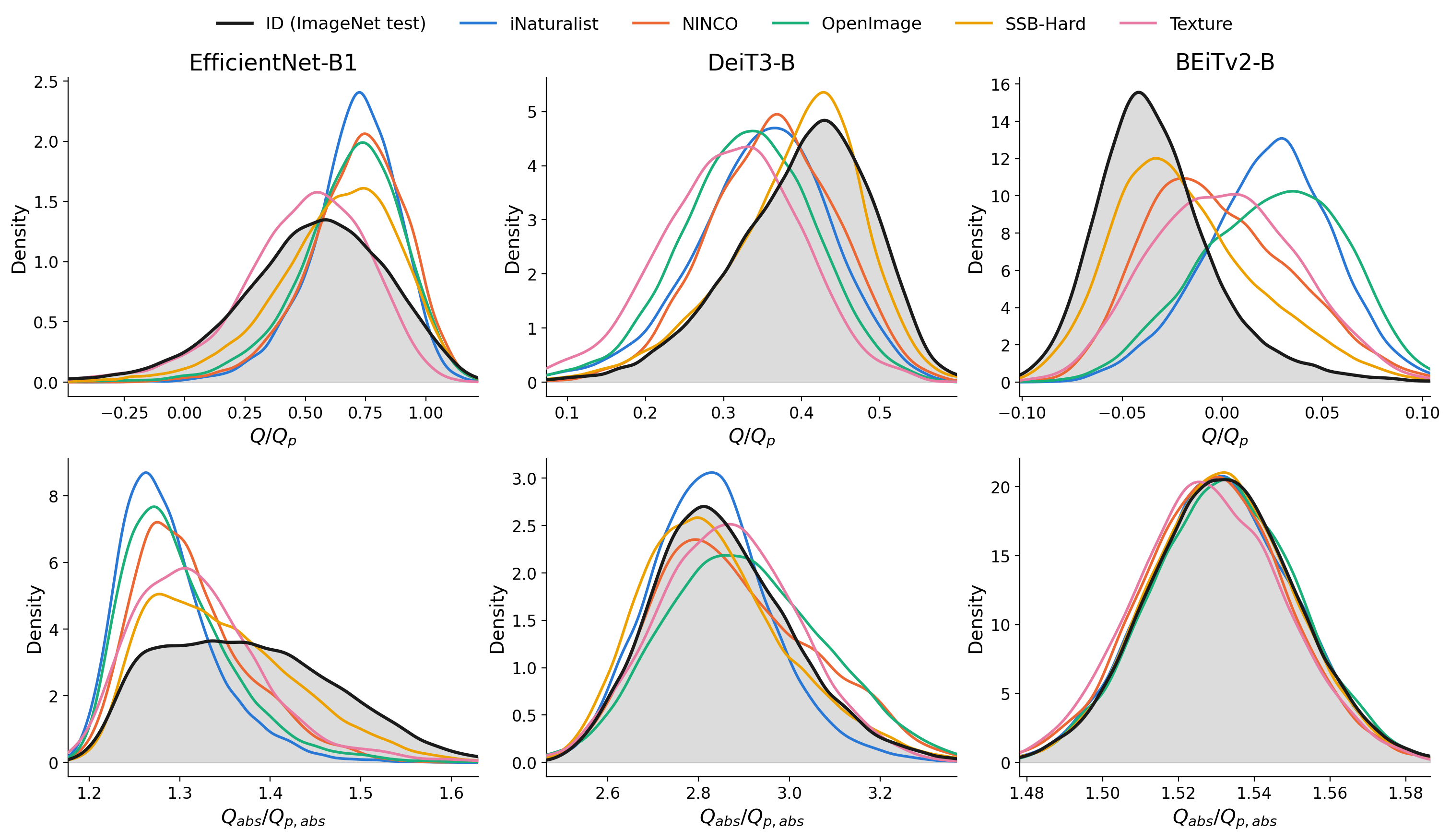}
\caption{Effect of the Absolute Fix on the separability of $r$ across six backbones with negative penultimate activations. In each panel, the top row shows the original ratio $r=Q/Q_p$ and the bottom row shows $r_{\text{abs}}=Q_{\text{abs}}/Q_{p,\text{abs}}$ after applying the Absolute Fix. The top block covers ViT-B/16, Swin-B, and ConvNeXt-B; the bottom block covers EfficientNet-B1, DeiT3-B, and BEiTv2-B.}
\label{fig:qqp_abs_before_after}
\end{figure}

Table~\ref{tab:abs_fix_ablation_performance} compares the Absolute Fix against Energy, SCALE, and SCALE+ReLU to verify its performance gain on a per-dataset basis. Across all seven backbones, SCALE+Abs achieves the best average AUROC and FPR@95, except on DeiT3-B, where its average AUROC ties with the original SCALE. On several backbones, most notably SwinV2-B, BEiTv2-B and EfficientNet-B1, the original SCALE performs worse than the plain Energy baseline; on SwinV2-B the degradation is most severe, with SCALE trailing Energy by over 21 points in average AUROC. On BEiTv2-B, SCALE+ReLU likewise fails to outperform Energy. These results underscore the need for a correction on architectures with negative penultimate activations, and validate that the Absolute Fix consistently yields a genuine improvement even in cases where SCALE or the ReLU-based correction fails to surpass the basic Energy baseline.

\begin{table}[!h]
\centering
\caption{Comparison of the Absolute Fix against Energy and SCALE baselines on backbones with negative penultimate activations, evaluated under the OpenOOD v1.5 protocol on ImageNet. Results are reported as AUROC/FPR@95 (\%); the best value in each column within each backbone block is highlighted in \textbf{bold}. Notably, on some backbones the original SCALE underperforms the plain Energy baseline, motivating the need for a correction.}
\label{tab:abs_fix_ablation_performance}
\resizebox{\textwidth}{!}{%
\begin{tabular}{llcccccc}
\toprule
\multirow{2}{*}{Backbone} & \multirow{2}{*}{Method} & \multicolumn{2}{c}{Near-OOD} & \multicolumn{3}{c}{Far-OOD} & \multirow{2}{*}{Avg} \\
\cmidrule(lr){3-4} \cmidrule(lr){5-7}
 & & NINCO & SSB-Hard & iNaturalist & Texture & OpenImage & \\
\midrule
\multirow{4}{*}{Swin-B}
 & Energy (Baseline) & 64.28/79.82 & 59.27/86.00 & 67.79/73.19 & 69.28/66.29 & 60.40/76.75 & 64.20/76.41 \\
 & SCALE              & 62.51/89.69 & 55.97/92.86 & 63.10/91.18 & 69.60/79.27 & 66.71/84.48 & 63.58/87.50 \\
 & SCALE+ReLU         & 73.90/76.51 & 63.09/\textbf{85.74} & 83.81/60.23 & 78.65/\textbf{65.18} & 80.83/66.83 & 76.06/70.90 \\
 & SCALE+Abs (Ours)   & \textbf{75.08}/\textbf{74.94} & \textbf{63.21}/86.24 & \textbf{86.82}/\textbf{53.16} & \textbf{80.08}/65.27 & \textbf{82.35}/\textbf{65.52} & \textbf{77.51}/\textbf{69.03} \\
\midrule
\multirow{4}{*}{SwinV2-B}
 & Energy (Baseline) & 72.14/72.88 & 66.41/81.83 & 77.87/62.67 & 76.25/\textbf{59.72} & 71.44/67.08 & 72.82/68.84 \\
 & SCALE              & 49.15/91.66 & 44.58/95.98 & 50.23/90.69 & 64.20/78.99 & 48.04/89.70 & 51.24/89.40 \\
 & SCALE+ReLU         & 74.10/74.45 & 64.13/83.43 & 88.81/44.82 & 75.47/69.31 & 81.10/62.42 & 76.72/66.88 \\
 & SCALE+Abs (Ours)   & \textbf{79.80}/\textbf{67.88} & \textbf{67.52}/\textbf{80.80} & \textbf{92.59}/\textbf{35.20} & \textbf{79.09}/68.19 & \textbf{86.24}/\textbf{56.52} & \textbf{81.05}/\textbf{61.72} \\
\midrule
\multirow{4}{*}{ViT-B/16}
 & Energy (Baseline) & 66.05/78.45 & 59.16/88.16 & 79.26/64.06 & 79.32/58.46 & 76.49/65.05 & 72.05/70.83 \\
 & SCALE              & 61.34/79.94 & 56.78/88.75 & 73.62/69.13 & 77.31/59.36 & 72.51/67.33 & 68.31/72.90 \\
 & SCALE+ReLU         & 74.55/73.15 & 64.17/85.46 & 87.93/48.63 & 84.08/53.32 & 82.91/59.14 & 78.73/63.94 \\
 & SCALE+Abs (Ours)   & \textbf{78.37}/\textbf{70.60} & \textbf{65.77}/\textbf{84.90} & \textbf{90.63}/\textbf{43.63} & \textbf{87.05}/\textbf{49.04} & \textbf{86.11}/\textbf{57.35} & \textbf{81.59}/\textbf{61.10} \\
\midrule
\multirow{4}{*}{ConvNeXt-B}
 & Energy (Baseline) & 64.43/82.68 & 65.67/84.36 & 55.97/90.24 & 49.12/88.40 & 52.95/89.38 & 57.63/87.01 \\
 & SCALE              & 67.77/90.66 & 61.35/92.97 & 57.62/98.73 & 69.53/86.10 & 65.35/92.05 & 64.32/92.10 \\
 & SCALE+ReLU         & \textbf{81.77}/\textbf{65.21} & 70.48/\textbf{80.44} & 90.06/45.95 & 72.26/74.10 & 83.54/63.28 & 79.62/65.80 \\
 & SCALE+Abs (Ours)   & 81.07/67.27 & \textbf{72.27}/81.64 & \textbf{91.86}/\textbf{39.35} & \textbf{76.81}/\textbf{70.30} & \textbf{85.25}/\textbf{61.85} & \textbf{81.45}/\textbf{64.08} \\
\midrule
\multirow{4}{*}{DeiT3-B}
 & Energy (Baseline) & 74.91/82.65 & 68.70/87.96 & 84.76/67.37 & 80.56/63.05 & 79.63/74.06 & 77.71/75.02 \\
 & SCALE              & 77.59/75.48 & 66.36/86.23 & 87.45/53.05 & \textbf{85.77}/\textbf{51.60} & \textbf{85.27}/59.08 & \textbf{80.49}/65.09 \\
 & SCALE+ReLU         & 76.53/75.59 & \textbf{70.63}/83.54 & 88.55/50.64 & 81.58/61.60 & 83.65/59.96 & 80.19/66.26 \\
 & SCALE+Abs (Ours)   & \textbf{78.20}/\textbf{72.03} & 69.18/\textbf{83.12} & \textbf{89.03}/\textbf{45.50} & 82.25/54.57 & 83.81/\textbf{57.43} & \textbf{80.49}/\textbf{62.53} \\
\midrule
\multirow{4}{*}{EfficientNet-B1}
 & Energy (Baseline) & 66.12/87.77 & 63.51/88.10 & 70.54/87.57 & 61.84/89.45 & 70.99/83.52 & 66.60/87.28 \\
 & SCALE              & 43.47/99.52 & 48.47/97.87 & 49.45/99.59 & 62.85/93.87 & 49.71/98.48 & 50.79/97.87 \\
 & SCALE+ReLU         & 63.63/82.78 & 62.82/\textbf{85.00} & 71.70/72.73 & 68.81/73.32 & 69.53/73.88 & 67.30/77.54 \\
 & SCALE+Abs (Ours)   & \textbf{77.73}/\textbf{76.13} & \textbf{65.92}/85.79 & \textbf{90.15}/\textbf{44.25} & \textbf{83.76}/\textbf{53.67} & \textbf{86.90}/\textbf{54.12} & \textbf{80.89}/\textbf{62.79} \\
\midrule
\multirow{4}{*}{BEiTv2-B}
 & Energy (Baseline) & 74.76/69.28 & 67.64/80.40 & 86.88/42.58 & 82.65/50.12 & 81.68/52.76 & 78.72/59.03 \\
 & SCALE              & 62.86/79.94 & 60.79/84.78 & 72.30/66.23 & 73.88/63.85 & 64.38/74.96 & 66.84/73.95 \\
 & SCALE+ReLU         & 72.13/70.55 & 65.16/82.17 & 82.41/51.07 & 79.70/52.71 & 77.52/58.28 & 75.38/62.96 \\
 & SCALE+Abs (Ours)   & \textbf{78.40}/\textbf{64.29} & \textbf{69.57}/\textbf{78.63} & \textbf{88.97}/\textbf{37.39} & \textbf{84.25}/\textbf{48.07} & \textbf{84.68}/\textbf{47.33} & \textbf{81.17}/\textbf{55.14} \\
\bottomrule
\end{tabular}%
}
\end{table}

\clearpage

\section{Visualization of the Anti-coherence on ReLU-based Backbones}
\label{appendix:vis_relu}

Figure~\ref{fig:anti_coh_relu} visualizes the Anti-coherence term $A(\mathbf{z})$ on ReLU-based backbones where OOD samples also tend to exhibit higher $A(\mathbf{z})$ values than ID samples. Nevertheless the performance gain over SCALE is comparatively modest on these backbones (see Table~\ref{tab:scale_absfix_anticoh_progression}). We attribute this to the fact that the scaling factor already provides sufficient separability under the ImageNet ID setting as discussed in Section~\ref{subsec:main_results}. Hence, on backbones with only non-negative activations there is little room for the Anti-coherence term to contribute additional signal.

\begin{figure}[!h]
\centering
\includegraphics[width=\textwidth]{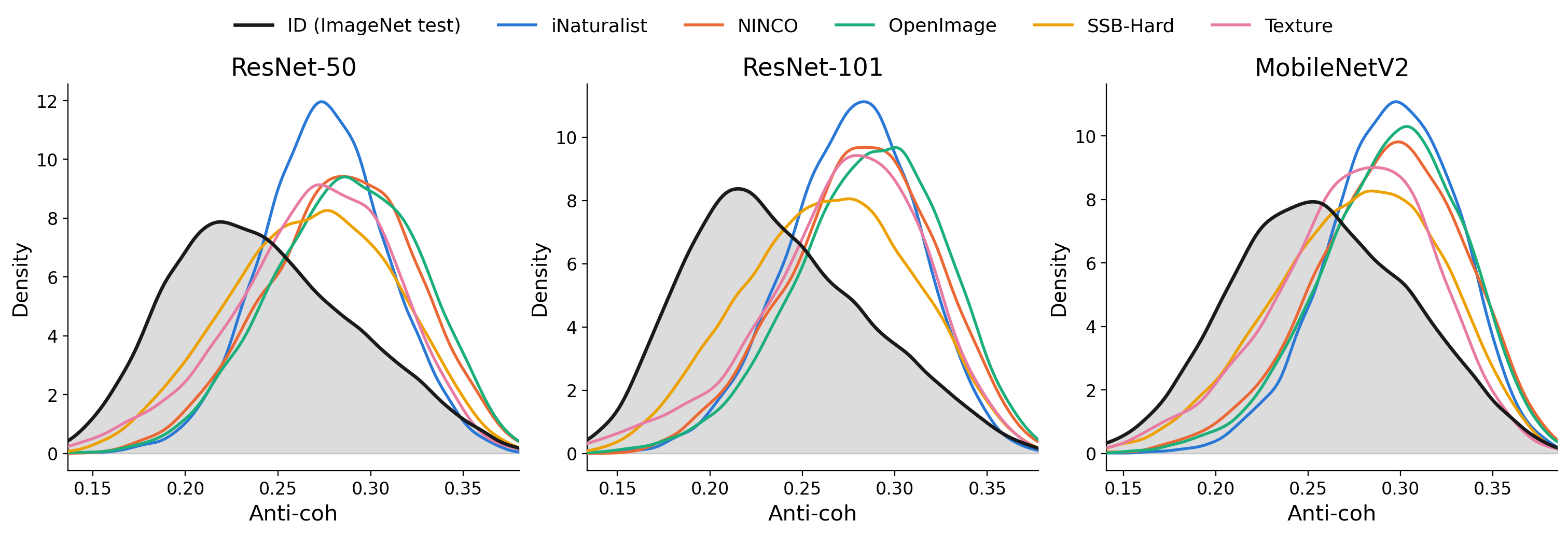}
\caption{Distribution of the Anti-coherence term $A(z)$ for ID and OOD samples on backbones with non-negative penultimate activations due to ReLU.}
\label{fig:anti_coh_relu}
\end{figure}

\section{Comparison with Representation-based Methods}
\label{appendix:Rep_method}

Table~\ref{tab:vs_representation} compares our method against representation-based methods that require access to training statistics. On architectures with non-negative activations (ResNet-50, ResNet-101, and MobileNetV2), our training-data-free method already achieves the best performance among all methods compared. On architectures with negative penultimate activations, however, our method still trails the strongest representation-based methods, most notably CADRef and Mahalanobis. This suggests that activation-shaping methods including ours remain fundamentally strongest on architectures such as ResNet, and while our use of sign information yields substantial improvements it still falls short of representation-based methods on architectures with negative activations. Therefore, closing this gap on negative-activation architectures without relying on training data remains an important direction for future activation-shaping methods.

\begin{table}[h]
\centering
\caption{OOD detection performance (AUROC/FPR@95, \%) of Anti-coh (Ours) against representation-based methods that require training statistics, averaged over five OOD datasets on ImageNet. The best value in each column is highlighted in \textbf{bold}.}
\label{tab:vs_representation}
\resizebox{\textwidth}{!}{%
\begin{tabular}{lcccccccccc}
\toprule
Method & ResNet-50 & ResNet-101 & Swin-B & SwinV2-B & ViT-B/16 & ConvNeXt-B & DeiT3-B & EfficientNet-B1 & MobileNetV2 & BEiTv2-B \\
\midrule
Mahalanobis & 66.58/83.64 & 72.49/78.77 & 86.74/56.72 & \textbf{87.99}/\textbf{50.26} & \textbf{86.77}/\textbf{48.91} & 87.30/53.67 & 86.76/56.74 & 83.44/65.57 & 51.56/90.07 & 88.67/46.06 \\
k-NN        & 81.69/59.76 & 83.83/55.75 & 85.67/60.38 & 86.07/62.60 & 84.01/64.05 & 86.61/59.82 & 83.69/71.19 & 82.16/68.94 & 76.18/69.52 & 87.83/48.82 \\
NCI         & 84.31/61.82 & 84.89/60.49 & 84.00/62.16 & 85.17/58.20 & 82.29/66.46 & 85.76/58.07 & 83.48/62.28 & \textbf{84.23}/68.59 & 79.67/71.95 & 86.56/50.86 \\
fDBD        & 85.96/58.40 & 86.44/57.01 & 84.06/62.33 & 85.42/57.24 & 82.51/65.62 & 86.40/55.98 & 85.06/58.33 & 82.90/70.35 & 84.00/63.48 & 86.28/51.22 \\
NNGuide     & 86.50/50.15 & 87.59/48.21 & 86.74/56.67 & 86.86/58.60 & 84.56/61.95 & 86.92/58.20 & 82.09/65.57 & 81.68/70.12 & 82.91/58.63 & 88.31/45.95 \\
CADRef      & 89.43/40.71 & 89.67/41.02 & \textbf{87.24}/\textbf{54.27} & 87.85/51.07 & 86.24/57.89 & \textbf{88.07}/\textbf{52.20} & \textbf{87.28}/\textbf{52.87} & 79.35/76.31 & 86.60/47.91 & \textbf{89.47}/\textbf{42.12} \\
Anti-coh (Ours) & \textbf{90.65}/\textbf{38.65} & \textbf{90.62}/\textbf{38.46} & 83.29/61.39 & 85.06/56.00 & 82.92/63.18 & 86.32/57.33 & 83.44/57.21 & 82.52/\textbf{57.48} & \textbf{87.27}/\textbf{46.10} & 86.00/51.82 \\
\bottomrule
\end{tabular}%
}
\end{table}

\clearpage

\section{Experiments of Hyperparameter Sensitivity}
\label{appendix:hyp}

Figure~\ref{fig:hyperparam_sensitivity} shows how AUROC on iNaturalist, OpenImage, and Texture responds to varying $\lambda$ across six backbones. On the four backbones with substantial negative activation mass (Swin-B, ViT-B/16, ConvNeXt-B, and BEiTv2-B) AUROC increases with $\lambda$ and then changes only gradually, indicating that our method is not overly sensitive to the exact choice of $\lambda$ within this range. On EfficientNet-B1 and MobileNetV2 whose penultimate activations are predominantly non-negative, AUROC peaks at a moderate $\lambda$ and then declines, most notably on Texture. This suggests that excessive regularization can be detrimental on backbones dominated by positive activations. 

\begin{figure}[!h]
\centering
\includegraphics[width=\textwidth]{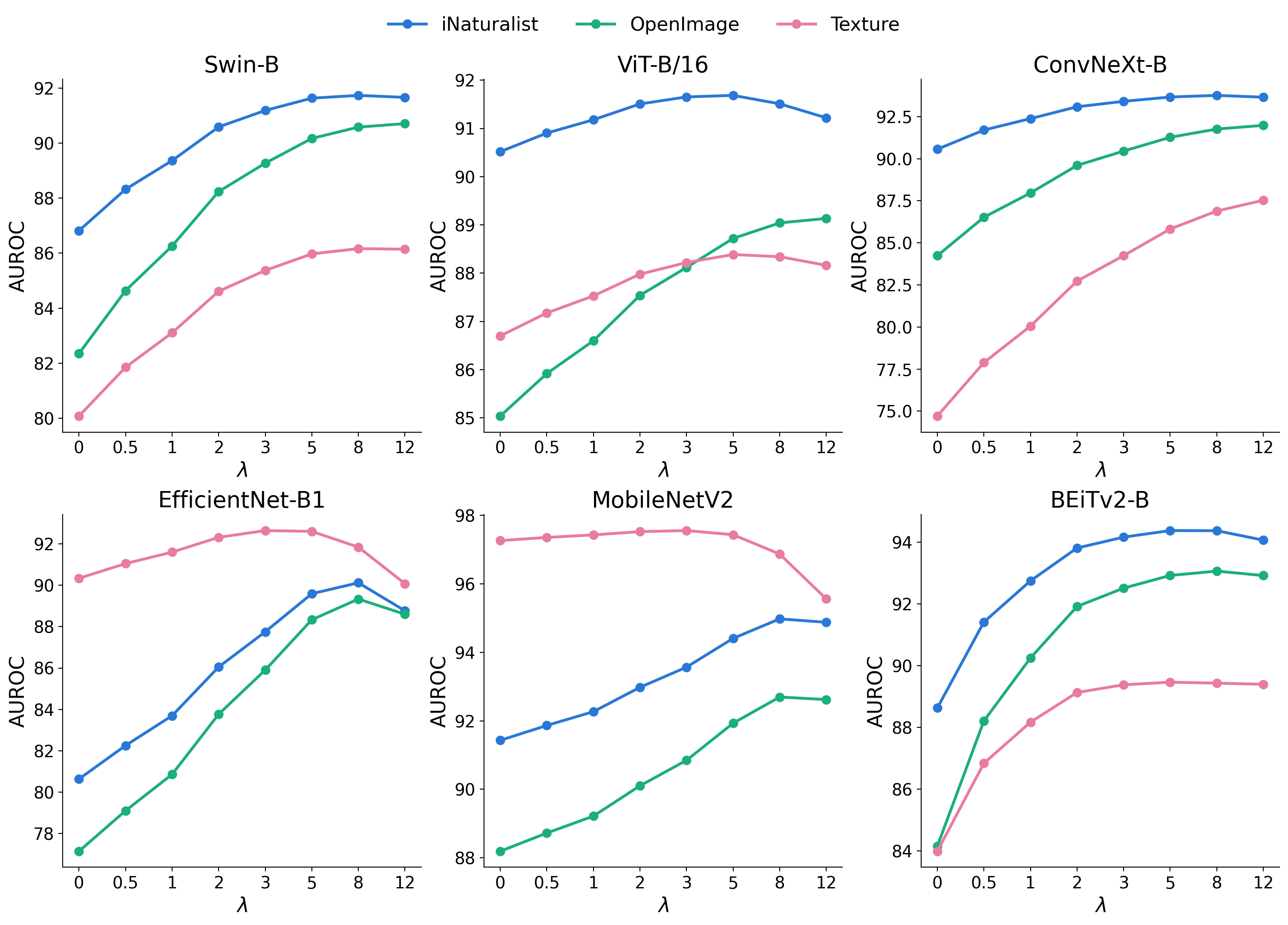}
\caption{Sensitivity of AUROC to the Anti-coherence regularization strength $\lambda$ across six backbones. Swin-B, ViT-B/16, ConvNeXt-B, and BEiTv2-B have substantial negative activation mass whereas EfficientNet-B1 and MobileNetV2 have predominantly non-negative penultimate activations. The percentile $p$ is fixed to the value selected in Table~\ref{tab:hyperparam_selected} for all $\lambda$ and $\lambda = 0$ corresponds to SCALE+Abs at this percentile.}
\label{fig:hyperparam_sensitivity}
\end{figure}

\clearpage

\section{Implementation Details}
\label{appendix:implementation}

\subsection{Method Implementations}

All methods are implemented in PyTorch 2.4.1~\citep{paszke2019pytorch}, and all experiments are conducted on a machine equipped with an AMD Ryzen 9 7950X CPU and an NVIDIA RTX 4090 GPU. Below we describe our implementation of each compared method, along with its hyperparameters and how they are set.

\textbf{MSP.} MSP~\citep{hendrycks2017a} scores a sample by the largest post-softmax class probability. Using the logit $\boldsymbol{\ell} = \mathbf{W}\phi(\mathbf{x}) + \mathbf{b}$, where $\mathbf{W} \in \mathbb{R}^{C \times d}$ and $\mathbf{b} \in \mathbb{R}^C$ denote the classifier head weight matrix and bias vector, the score is given by
\begin{equation}
    S(\mathbf{x}) = \max_c\, \mathrm{softmax}(\boldsymbol{\ell})_c.
\end{equation}

\textbf{MaxLogit.} MaxLogit~\citep{hendrycks2022scaling} instead takes the maximum component of the raw logit directly, without the softmax normalization:
\begin{equation}
    S(\mathbf{x}) = \max_c\, \ell_c.
\end{equation}

\textbf{Energy.} Rather than relying on a single class, the energy score~\citep{liu2020energy} aggregates all logit components through a log-sum-exp:
\begin{equation}
    S(\mathbf{x}) = T \cdot \log \sum_c \exp(\ell_c/T),
\end{equation}
where the temperature $T$ is fixed to $1$ throughout.

\textbf{ASH-S.} ASH-S~\citep{djurisic2023extremely} prunes activations below a threshold $\tau$ and applies exponential rescaling to the surviving components:
\begin{equation}
    \phi_{\mathrm{ASH}}(\mathbf{x}) = \phi(\mathbf{x}) \cdot \mathbf{1}[\phi(\mathbf{x}) > \tau] \cdot \exp\!\left(\frac{\sum_i \phi_i(\mathbf{x})}{\sum_{i:\, \phi_i(\mathbf{x}) > \tau} \phi_i(\mathbf{x})}\right),
\end{equation}
where $\tau = P_p(\phi(\mathbf{x}))$ denotes the $p$-th percentile of $\phi(\mathbf{x})$, with $p$ as the sole hyperparameter. The shaped activation is then passed through the classifier head and scored using the energy score,
\begin{equation}
    S(\mathbf{x}) = T \cdot \log \sum_c \exp\!\left(\frac{(\mathbf{W}\phi_{\mathrm{ASH}}(\mathbf{x}) + \mathbf{b})_c}{T}\right).
\end{equation}

\textbf{SCALE.} SCALE~\citep{xu2024scaling} applies the same exponential rescaling as ASH-S but without the pruning step:
\begin{equation}
    \phi_{\mathrm{SCALE}}(\mathbf{x}) = \phi(\mathbf{x}) \cdot \exp\!\left(\frac{\sum_i \phi_i(\mathbf{x})}{\sum_{i:\, \phi_i(\mathbf{x}) > \tau} \phi_i(\mathbf{x})}\right),
\end{equation}
where $\tau = P_p(\phi(\mathbf{x}))$ is again the $p$-th percentile of $\phi(\mathbf{x})$, with $p$ as the sole hyperparameter. The shaped activation is then passed through the classifier head and scored using the energy score,
\begin{equation}
    S(\mathbf{x}) = T \cdot \log \sum_c \exp\!\left(\frac{(\mathbf{W}\phi_{\mathrm{SCALE}}(\mathbf{x}) + \mathbf{b})_c}{T}\right).
\end{equation}

\textbf{LTS.} LTS~\citep{djurisic2025logit} rectifies the penultimate activation with ReLU before computing SCALE's percentile ratio, and applies the square of the resulting factor to the logit rather than to the activation itself, with the percentile threshold $p$ serving as the sole hyperparameter. With $\phi_+(\mathbf{x}) = \mathrm{ReLU}(\phi(\mathbf{x}))$, the scaling factor is
\begin{equation}
    r_{\mathrm{LTS}} = \frac{\sum_i \phi_{+,i}(\mathbf{x})}{\sum_{i:\, \phi_{+,i}(\mathbf{x}) > P_p(\phi_+(\mathbf{x}))} \phi_{+,i}(\mathbf{x})},
\end{equation}
where $P_p(\phi_+(\mathbf{x}))$ denotes the $p$-th percentile of $\phi_+(\mathbf{x})$. The original, unrectified activation is passed through the classifier head to obtain the logit $\boldsymbol{\ell} = \mathbf{W}\phi(\mathbf{x}) + \mathbf{b}$, which is then scaled by $r_{\mathrm{LTS}}^2$ and scored using the energy score,
\begin{equation}
    S(\mathbf{x}) = T \cdot \log \sum_c \exp\!\left(\frac{r_{\mathrm{LTS}}^2 \cdot \ell_c}{T}\right).
\end{equation}

\textbf{ReAct.} ReAct~\citep{sun2021react} clips penultimate activations at a threshold $\tau$ set to the $p$-th percentile of the training features, prior to energy scoring. With the clipped logit $\boldsymbol{\ell} = \mathbf{W}\,\mathrm{clip}(\phi(\mathbf{x}), \tau) + \mathbf{b}$, the score is given by
\begin{equation}
    S(\mathbf{x}) = T \cdot \log \sum_c \exp(\ell_c/T),
\end{equation}
where $\mathrm{clip}(\phi(\mathbf{x}), \tau)$ truncates each activation at $\tau$ from above, and the percentile $p$ is the sole hyperparameter.

\textbf{VRA.} VRA~\citep{xu2023vra} rectifies each activation with a two-sided piecewise function, suppressing values below a lower bound $\alpha$ and clipping values above an upper bound $\beta$:
\begin{equation}
    \phi_{\mathrm{VRA}}(\mathbf{x})_i =
    \begin{cases}
        0 & \phi_i(\mathbf{x}) < \alpha \\
        \phi_i(\mathbf{x}) & \alpha \leq \phi_i(\mathbf{x}) \leq \beta \\
        \beta & \phi_i(\mathbf{x}) > \beta
    \end{cases}.
\end{equation}
The rectified activation is then passed through the classifier head and scored using the energy score,
\begin{equation}
    S(\mathbf{x}) = T \cdot \log \sum_c \exp\!\left(\frac{(\mathbf{W}\phi_{\mathrm{VRA}}(\mathbf{x}) + \mathbf{b})_c}{T}\right).
\end{equation}
Following the original implementation, $\alpha$ and $\beta$ are set to the $\eta_\alpha$- and $\eta_\beta$-quantiles of the pooled ID training activations, respectively, with $\eta_\alpha$ and $\eta_\beta$ selected as hyperparameters subject to $\eta_\alpha < \eta_\beta$.

\textbf{OptFS.} OptFS~\citep{zhao2024towards} learns a closed-form scaling vector from ID training statistics and applies it to the penultimate activation. Activations are grouped into $K$ equal-width bins over $[\alpha, \beta)$ where $\alpha$ and $\beta$ are the $0.1$-th and $99.9$-th percentile of the pooled ID training activations and $\mathrm{bin}(z)$ denotes the index of the bin containing value $z$. For each bin $k$ an importance score $I_k$ is computed using the ID training set as the average of $W_{c^*, i}\, \phi_i(\mathbf{x})$ over all components $i$ falling in that bin where $c^* = \arg\max_c (\mathbf{W}\phi(\mathbf{x}) + \mathbf{b})_c$ is the predicted class from the original unscaled logit. The scaling vector $\boldsymbol{\theta} \in \mathbb{R}^K$ is then obtained by normalizing $I \in \mathbb{R}^K$ in closed form as $\boldsymbol{\theta} = \sqrt{K}\, I / \|I\|_2$. At test time each activation is rescaled by the factor of its own bin,
\begin{equation}
    \phi_{\mathrm{OptFS}}(\mathbf{x})_i = \theta_{\mathrm{bin}(\phi_i(\mathbf{x}))} \cdot \phi_i(\mathbf{x}) \cdot \mathbf{1}[\alpha \leq \phi_i(\mathbf{x}) < \beta],
\end{equation}
and the OOD score is given by $S(\mathbf{x}) = \mathbf{W}_{c^*} \cdot \phi_{\mathrm{OptFS}}(\mathbf{x})$ where $c^*$ is again determined from the original unscaled logit. The bias term is omitted from the score following the official implementation. The number of bins $K$ is the sole hyperparameter while $\boldsymbol{\theta}$ is computed in closed form from ID training statistics without further tuning.

\textbf{LAPS.} LAPS~\citep{he2024exploring} clips each channel of the penultimate activation to a channel-specific asymmetric range, calibrated from how far that channel's ID training statistics deviate from the average across all channels. Let $\mu_i$ and $\sigma_i$ denote the per-channel mean and standard deviation over the ID training features, and let $\bar{\mu}$ and $\bar{\sigma}$ denote their averages across all channels. The clamping bounds for channel $i$ are
\begin{equation}
    \lambda_i^{(1)} = \lambda + m(\bar{\mu}-\mu_i) + n(\bar{\sigma}-\sigma_i), \qquad \lambda_i^{(2)} = \lambda - m(\bar{\mu}-\mu_i) + n(\bar{\sigma}-\sigma_i),
\end{equation}
and each activation is clipped as
\begin{equation}
    \phi_{\mathrm{LAPS}}(\mathbf{x})_i =
    \begin{cases}
        \mu_i - \lambda_i^{(2)}\sigma_i & \phi_i(\mathbf{x}) < \mu_i - \lambda_i^{(2)}\sigma_i \\
        \phi_i(\mathbf{x}) & \mu_i - \lambda_i^{(2)}\sigma_i \le \phi_i(\mathbf{x}) \le \mu_i + \lambda_i^{(1)}\sigma_i \\
        \mu_i + \lambda_i^{(1)}\sigma_i & \phi_i(\mathbf{x}) > \mu_i + \lambda_i^{(1)}\sigma_i
    \end{cases}.
\end{equation}
The clipped activation is then passed through the classifier head and scored using the energy score,
\begin{equation}
    S(\mathbf{x}) = T \cdot \log \sum_c \exp\!\left(\frac{(\mathbf{W}\phi_{\mathrm{LAPS}}(\mathbf{x}) + \mathbf{b})_c}{T}\right).
\end{equation}
The base level $\lambda$ and the deviation weights $m$ and $n$ are jointly tuned as hyperparameters.

\textbf{AdaSCALE-A/L.} AdaSCALE~\citep{regmi2026adascale} determines the percentile threshold used in SCALE's ratio adaptively per sample, by perturbing the input along a gradient direction and using the resulting shift in the pooled feature as an estimate of how OOD-like that sample is. This estimate is compared against a reference distribution built from the ID validation set to obtain a sample-specific percentile, which is then used in place of SCALE's fixed percentile. AdaSCALE-A applies the resulting factor to the feature before recomputing the logit, while AdaSCALE-L instead scales the original logit by the square of the factor; both are then scored using the energy score. The method has hyperparameters $\lambda$, $k_1$, $k_2$, $o$, and the percentile range $[p_{\min}, p_{\max}]$, which we set following Section 5.2 of the original paper: $(\lambda, k_1, k_2, o) = (10, 1\%, 5\%, 5\%)$ and $p_{\min} = 60$ are fixed, and only $p_{\max}$ is tuned as a hyperparameter.

\textbf{Anti-coh (Ours).} Our method computes the scaling factor $r_{\text{abs}} = Q_{\text{abs}}/Q_{p,\text{abs}}$ over the absolute-valued penultimate activation, and regularizes it with the Anti-coherence term $A(\mathbf{z})$, the fraction of the max logit's contribution that is sign-mismatched with the classifier weight. The predicted class $c^*$ used for $A(\mathbf{z})$ is the $\texttt{argmax}$ of the unscaled logit including the bias and the denominators of $r_{\text{abs}}$ and $A(\mathbf{z})$ are clamped at $10^{-8}$ for numerical stability. This yields the regularized scaling factor
\begin{equation}
    r_{\text{anti}} = r_{\text{abs}} - \lambda A(\mathbf{z}).
\end{equation}
The shaped activation is then given by
\begin{equation}
    \phi_{\text{anti}}(\mathbf{x}) = \exp\big(r_{\text{anti}}\big) \cdot \phi(\mathbf{x}),
\end{equation}
which is passed through the classifier head and scored using the energy score,
\begin{equation}
    S(\mathbf{x}) = T \cdot \log \sum_c \exp\!\left(\frac{(\mathbf{W}\phi_{\text{anti}}(\mathbf{x}) + \mathbf{b})_c}{T}\right).
\end{equation}
The percentile $p$ used to compute $r_{\text{abs}}$ and the regularization strength $\lambda$ are jointly tuned as hyperparameters. Although Lemma~\ref{lem:maxlogit_energy} assumes zero bias for the analysis, we include the classifier bias when combining the shaped activation with the energy score for consistency with the compared methods since excluding it did not affect performance. Equivalently $\phi_{\text{anti}}(\mathbf{x}) = \exp(-\lambda A(\mathbf{z})) \cdot \exp(r_{\text{abs}}) \cdot \phi(\mathbf{x})$ so that the regularization acts as a multiplicative damping factor $\exp(-\lambda A(\mathbf{z})) \in (0, 1]$ on the activation scale independently of the magnitude of $r_{\text{abs}}$.

\textbf{Backbones and checkpoints.} All ImageNet backbones are pretrained on ImageNet-1K only. DeiT3-B and BEiTv2-B are loaded from timm as \texttt{deit3\_base\_patch16\_224.fb\_in1k} and \texttt{beitv2\_base\_patch16\_224.in1k\_ft\_in1k} respectively. For CIFAR-10 and CIFAR-100 we use a single checkpoint released with OpenOOD v1.5 and report the results of this checkpoint without standard deviation.

\clearpage

\subsection{Hyperparameter Search and Selection}

Table~\ref{tab:hyperparam_search} lists the hyperparameter search space explored for each method. Table~\ref{tab:hyperparam_selected} and Table~\ref{tab:hyperparam_selected_cifar10} report the hyperparameters selected by validation AUROC on the validation splits provided by OpenOOD v1.5, for the ImageNet and CIFAR-10 ID settings, respectively. The validation splits provided by OpenOOD v1.5 contain both ID and OOD validation samples and no training data is used for hyperparameter selection.

\begin{table}[h]
\centering
\caption{Hyperparameters, search grids, and tuning procedures for all compared methods on the ImageNet experiments.}
\label{tab:hyperparam_search}
\resizebox{\textwidth}{!}{%
\begin{tabular}{lll}
\toprule
Method & Hyperparameter(s) & Search grid \\
\midrule
ReAct & threshold percentile $p$ & $\{0.90, 0.95, 0.99\}$ \\
ASH-S & pruning percentile $p$ & $\{0.65, 0.70, 0.75, 0.80, 0.85, 0.90, 0.95\}$ \\
SCALE & percentile $p$ & $\{0.65, 0.70, 0.75, 0.80, 0.85, 0.90, 0.95\}$ \\
VRA & $(\eta_\alpha, \eta_\beta)$ quantile pair & $\eta_\alpha \in \{0.5, 0.6, 0.65, 0.7\}$, $\eta_\beta \in \{0.80, 0.85, 0.90, 0.95, 0.99\}$ \\
OptFS & number of bins $K$ & $\{1, 3, 5, 10, 25, 50, 100\}$ \\
LTS & percentile $p$ & $\{0.65, 0.70, 0.75, 0.80, 0.85, 0.90, 0.95\}$ \\
LAPS & $(\lambda, m, n)$ jointly & $\lambda \in \{0.5, 1.0, 1.5, 2.0, 2.5, 3.0\}$, $m \in \{0.0, 0.3, 0.6\}$, $n \in \{0.5, 0.6, 0.7, 0.8, 0.9, 1.0\}$ \\
AdaSCALE-A / L & percentile pair $(p_{\min}, p_{\max})$ & $\{(60,65), (60,70), (60,75), (60,80), (60,85), (60,90), (60,95), (60,99)\}$ \\
Anti-coh (Ours) & $(p, \lambda)$ jointly & $p \in \{0.5, 0.55, \dots, 0.95\}$, $\lambda \in \{0, 0.5, 1, 2, 3, 5, 8, 12\}$ \\
\bottomrule
\end{tabular}%
}
\end{table}

\begin{table}[h]
\centering
\caption{Selected hyperparameters for each method across the ten ImageNet backbones, chosen via validation AUROC.}
\label{tab:hyperparam_selected}
\resizebox{\textwidth}{!}{%
\begin{tabular}{lcccccccccc}
\toprule
Method & RN-50 & RN-101 & Swin-B & SwinV2-B & ViT-B/16 & ConvNeXt-B & DeiT3-B & EfficientNet-B1 & MobileNetV2 & BEiTv2-B \\
\midrule
ReAct       & 0.95 & 0.95 & 0.90 & 0.90 & 0.90 & 0.90 & 0.95 & 0.95 & 0.90 & 0.90 \\
ASH-S       & 0.85 & 0.85 & 0.65 & 0.65 & 0.65 & 0.95 & 0.65 & 0.65 & 0.85 & 0.65 \\
SCALE       & 0.85 & 0.85 & 0.95 & 0.65 & 0.65 & 0.95 & 0.90 & 0.65 & 0.85 & 0.65 \\
VRA $(\eta_\alpha,\eta_\beta)$ & (0.7, 0.95) & (0.7, 0.90) & (0.5, 0.85) & (0.5, 0.80) & (0.5, 0.80) & (0.5, 0.80) & (0.7, 0.99) & (0.5, 0.95) & (0.6, 0.80) & (0.5, 0.90) \\
OptFS ($K$) & 100 & 100 & 3 & 100 & 100 & 3 & 10 & 5 & 100 & 50 \\
LTS         & 0.85 & 0.85 & 0.95 & 0.95 & 0.95 & 0.95 & 0.85 & 0.65 & 0.85 & 0.70 \\
LAPS $(\lambda,m,n)$ & (2.0, 0.6, 1.0) & (2.0, 0.6, 1.0) & (1.5, 0.6, 1.0) & (1.5, 0.6, 1.0) & (1.5, 0.0, 1.0) & (1.5, 0.6, 1.0) & (2.0, 0.0, 0.5) & (3.0, 0.0, 0.5) & (0.5, 0.0, 0.9) & (1.5, 0.6, 1.0) \\
AdaSCALE-A  & (60,75) & (60,75) & (60,99) & (60,99) & (60,70) & (60,99) & (60,80) & (60,99) & (60,75) & (60,80) \\
AdaSCALE-L  & (60,70) & (60,70) & (60,99) & (60,99) & (60,70) & (60,99) & (60,95) & (60,99) & (60,70) & (60,80) \\
Anti-coh $(p,\lambda)$ & (0.9, 3) & (0.9, 3) & (0.95, 12) & (0.95, 12) & (0.95, 12) & (0.9, 12) & (0.85, 5) & (0.95, 8) & (0.95, 8) & (0.75, 8) \\
\bottomrule
\end{tabular}%
}
\end{table}

\begin{table}[h]
\centering
\caption{Selected hyperparameters for each method on CIFAR-10 (ResNet-18), chosen via validation AUROC.}
\label{tab:hyperparam_selected_cifar10}
\resizebox{\textwidth}{!}{%
\begin{tabular}{ccccccccccc}
\toprule
ReAct & ASH-S & SCALE & VRA $(\eta_\alpha,\eta_\beta)$ & OptFS ($K$) & LTS & LAPS $(\lambda,m,n)$ & AdaSCALE-A & AdaSCALE-L & Anti-coh $(p,\lambda)$ \\
\midrule
0.99 & 0.65 & 0.65 & (0.5, 0.99) & 1 & 0.65 & (2.0, 0.6, 0.5) & (60, 65) & (60, 65) & (0.5, 12) \\
\bottomrule
\end{tabular}%
}
\end{table}

\clearpage

\section{Full Per-backbone Results}
\label{appendix:full_results}

Table~\ref{tab:cifar100_ablation} reports results on CIFAR-100 under the OpenOOD setting where most methods attain noticeably lower performance than on CIFAR-10. We suspect this stems from the limited number of training images per class which prevents the pretrained model from learning sufficiently discriminative penultimate activations. In the CIFAR-10 setting (Table~\ref{tab:cifar10_main_exp}) detection performance on MNIST, which we consider the simplest OOD dataset, is nearly saturated. On CIFAR-100, however, MNIST detection performance is substantially lower across every method despite CIFAR-100 not containing any digit-related classes. We take this as suggestive evidence that the activation itself fails to carry adequate feature information rather than any inherent difficulty in distinguishing MNIST from a specific ID class. Additionally, Tables~\ref{tab:resnet50_ablation} through~\ref{tab:beitv2b_ablation} report per-backbone results on each OOD dataset under the ImageNet ID setting.

\begin{table}[htbp]
\centering
\caption{OOD detection performance (AUROC/FPR@95, \%) on CIFAR-100 (ResNet-18) under the OpenOOD v1.5 protocol. Methods are grouped by resource requirement. The best value in each column is highlighted in \textbf{bold}.}
\label{tab:cifar100_ablation}
\resizebox{\textwidth}{!}{%
\begin{tabular}{lccccccc}
\toprule
Method & MNIST & Places365 & SVHN & Texture & TIN & CIFAR-10 & Avg \\
\midrule
\multicolumn{8}{c}{\textit{Training-data-dependent}} \\
\midrule
ReAct       & 77.47/84.48 & 77.51/80.66 & 82.48/76.10 & 80.55/77.57 & 80.63/75.02 & 78.71/79.76 & 79.56/78.93 \\
VRA         & 77.95/83.72 & 77.68/80.41 & 83.10/74.37 & 81.23/76.70 & 80.77/74.91 & 78.78/79.64 & 79.92/78.29 \\
OptFS       & 76.86/86.46 & 77.39/80.41 & 81.79/77.36 & 80.20/78.51 & 80.49/74.92 & 78.49/79.58 & 79.20/79.54 \\
LAPS        & 75.73/85.34 & 77.92/81.22 & 81.82/77.20 & 81.67/77.00 & 80.76/75.49 & 78.49/80.46 & 79.40/79.45 \\
\midrule
\multicolumn{8}{c}{\textit{Gradient-dependent}} \\
\midrule
AdaSCALE-A  & 79.96/82.36 & 78.21/78.74 & 85.53/68.10 & 81.23/74.63 & 80.52/75.61 & \textbf{79.28}/79.19 & 80.79/76.44 \\
AdaSCALE-L  & \textbf{81.76}/77.96 & \textbf{78.57}/78.05 & \textbf{88.09}/\textbf{58.16} & \textbf{82.56}/70.20 & 79.93/76.36 & 78.92/79.48 & \textbf{81.64}/\textbf{73.37} \\
\midrule
\multicolumn{8}{c}{\textit{Training-free, gradient-free}} \\
\midrule
MSP         & 74.45/88.53 & 77.17/80.24 & 79.37/79.29 & 78.07/79.61 & 80.11/75.54 & 78.54/79.97 & 77.95/80.53 \\
MaxLogit    & 77.82/86.21 & 77.36/80.20 & 82.33/77.22 & 79.36/78.72 & 80.69/74.71 & 79.20/79.40 & 79.46/79.41 \\
Energy      & 78.22/84.45 & 77.12/80.51 & 82.67/76.03 & 79.33/77.91 & 80.51/74.91 & 79.01/79.66 & 79.48/78.91 \\
ASH-S       & 80.08/80.06 & 78.12/78.68 & 85.84/66.58 & 81.79/72.80 & 80.23/75.29 & 79.17/79.34 & 80.87/75.46 \\
SCALE       & 79.62/81.63 & 78.09/78.35 & 85.14/68.94 & 81.28/74.29 & 80.47/74.66 & 79.24/\textbf{79.16} & 80.64/76.17 \\
LTS         & 80.91/\textbf{77.67} & 78.30/\textbf{77.79} & 87.14/61.76 & 82.49/\textbf{69.95} & 79.87/75.27 & 78.85/79.46 & 81.26/73.65 \\
Anti-coh (Ours) & 77.34/86.86 & 77.42/80.85 & 80.39/79.05 & 80.42/77.71 & \textbf{81.61}/\textbf{73.84} & 79.12/80.28 & 79.38/79.77 \\
\bottomrule
\end{tabular}%
}
\end{table}

\begin{table}[htbp]
\centering
\caption{OOD detection performance (AUROC/FPR@95, \%) on ResNet-50 (ImageNet) under the OpenOOD v1.5 protocol. Methods are grouped by resource requirement. The best value in each column is highlighted in \textbf{bold}.}
\label{tab:resnet50_ablation}
\resizebox{\textwidth}{!}{%
\begin{tabular}{lcccccc}
\toprule
Method & iNaturalist & NINCO & OpenImage & SSB-Hard & Texture & Avg \\
\midrule
\multicolumn{7}{c}{\textit{Training-data-dependent}} \\
\midrule
ReAct       & 96.32/19.96 & 81.71/70.43 & 91.88/41.33 & 73.08/80.47 & 91.62/41.44 & 86.92/50.72 \\
VRA         & 96.79/17.48 & 82.23/70.28 & 93.08/36.66 & 71.30/83.94 & 95.05/24.29 & 87.69/46.53 \\
OptFS       & 96.62/18.43 & 82.45/70.48 & 92.55/39.26 & 71.11/83.79 & 95.47/24.91 & 87.64/47.38 \\
LAPS        & 95.09/27.80 & 80.80/72.42 & 91.94/42.26 & 70.96/83.22 & 92.89/35.05 & 86.34/52.15 \\
\midrule
\multicolumn{7}{c}{\textit{Gradient-dependent}} \\
\midrule
AdaSCALE-A  & 95.26/27.99 & 84.13/66.52 & 94.04/33.51 & 77.96/77.86 & 92.84/32.41 & 88.85/47.66 \\
AdaSCALE-L  & 96.48/19.78 & 85.00/63.70 & \textbf{94.70}/29.39 & \textbf{78.16}/\textbf{77.19} & 94.69/23.90 & 89.81/42.79 \\
\midrule
\multicolumn{7}{c}{\textit{Training-free, gradient-free}} \\
\midrule
MSP         & 88.43/52.73 & 79.98/75.94 & 84.99/63.61 & 72.16/84.49 & 80.47/66.15 & 81.21/68.59 \\
MaxLogit    & 91.15/50.82 & 80.41/76.61 & 89.26/57.13 & 72.75/83.70 & 86.41/54.20 & 84.00/64.49 \\
Energy      & 90.61/53.83 & 79.70/77.56 & 89.15/57.09 & 72.35/83.82 & 86.74/52.23 & 83.71/64.91 \\
ASH-S       & 97.70/12.33 & 84.42/63.51 & 93.83/30.66 & 74.59/80.89 & \textbf{97.11}/\textbf{13.60} & 89.53/40.20 \\
SCALE       & 98.00/10.52 & 85.27/61.79 & 93.97/29.92 & 77.23/78.00 & 96.73/14.65 & 90.24/38.97 \\
LTS         & 97.85/11.23 & 85.23/61.77 & 94.03/30.12 & 77.04/78.08 & 96.43/15.67 & 90.12/39.38 \\
Anti-coh (Ours) & \textbf{98.08}/\textbf{10.04} & \textbf{86.44}/\textbf{60.77} & 94.46/\textbf{29.33} & 77.64/77.86 & 96.61/15.27 & \textbf{90.65}/\textbf{38.65} \\
\bottomrule
\end{tabular}%
}
\end{table}

\begin{table}[htbp]
\centering
\caption{OOD detection performance (AUROC/FPR@95, \%) on ResNet-101 (ImageNet) under the OpenOOD v1.5 protocol. Methods are grouped by resource requirement. The best value in each column is highlighted in \textbf{bold}.}
\label{tab:resnet101_ablation}
\resizebox{\textwidth}{!}{%
\begin{tabular}{lcccccc}
\toprule
Method & iNaturalist & NINCO & OpenImage & SSB-Hard & Texture & Avg \\
\midrule
\multicolumn{7}{c}{\textit{Training-data-dependent}} \\
\midrule
ReAct       & 95.62/23.05 & 83.27/68.07 & 92.56/39.51 & 74.41/78.27 & 92.40/36.99 & 87.65/49.18 \\
VRA         & 96.54/18.28 & 82.90/67.37 & 93.65/33.35 & 71.45/82.65 & 96.10/19.49 & 88.13/44.23 \\
OptFS       & 95.67/24.19 & 83.28/68.71 & 92.97/37.36 & 72.31/82.22 & 95.78/23.78 & 88.00/47.25 \\
LAPS        & 94.69/29.03 & 82.80/69.39 & 92.55/39.67 & 73.11/80.59 & 93.67/31.10 & 87.37/49.96 \\
\midrule
\multicolumn{7}{c}{\textit{Gradient-dependent}} \\
\midrule
AdaSCALE-A  & 93.86/34.17 & 85.09/62.45 & 94.43/31.51 & 78.50/75.78 & 93.37/30.02 & 89.05/46.79 \\
AdaSCALE-L  & 95.19/26.85 & 86.00/60.45 & \textbf{95.14}/27.51 & \textbf{78.70}/\textbf{75.57} & 95.14/22.06 & 90.03/42.49 \\
\midrule
\multicolumn{7}{c}{\textit{Training-free, gradient-free}} \\
\midrule
MSP         & 86.24/57.02 & 80.66/74.55 & 84.91/62.62 & 72.61/83.66 & 81.51/62.57 & 81.19/68.08 \\
MaxLogit    & 88.78/55.81 & 82.31/72.49 & 89.36/55.07 & 74.60/80.77 & 87.57/48.30 & 84.52/62.49 \\
Energy      & 88.23/58.34 & 81.81/73.43 & 89.24/55.57 & 74.38/80.58 & 87.87/47.34 & 84.31/63.05 \\
ASH-S       & 96.65/18.16 & 85.30/61.04 & 94.36/28.09 & 74.55/79.99 & \textbf{97.55}/\textbf{11.90} & 89.68/39.83 \\
SCALE       & 97.01/15.69 & 86.12/59.63 & 94.50/27.99 & 77.06/76.95 & 97.17/13.39 & 90.37/38.73 \\
LTS         & 96.79/16.71 & 86.18/59.19 & 94.54/27.97 & 77.24/76.54 & 96.94/14.15 & 90.34/38.91 \\
Anti-coh (Ours) & \textbf{97.20}/\textbf{14.62} & \textbf{86.99}/\textbf{59.14} & 95.04/\textbf{26.90} & 76.90/77.53 & 96.98/14.13 & \textbf{90.62}/\textbf{38.46} \\
\bottomrule
\end{tabular}%
}
\end{table}

\begin{table}[htbp]
\centering
\caption{OOD detection performance (AUROC/FPR@95, \%) on Swin-B (ImageNet) under the OpenOOD v1.5 protocol. Methods are grouped by resource requirement. The best value in each column is highlighted in \textbf{bold}.}
\label{tab:swinb_ablation}
\resizebox{\textwidth}{!}{%
\begin{tabular}{lcccccc}
\toprule
Method & iNaturalist & NINCO & OpenImage & SSB-Hard & Texture & Avg \\
\midrule
\multicolumn{7}{c}{\textit{Training-data-dependent}} \\
\midrule
ReAct       & 88.60/58.96 & 80.24/75.16 & 85.83/62.41 & 68.61/84.90 & 84.05/58.87 & 81.47/68.06 \\
VRA         & 81.11/76.37 & 68.84/87.31 & 81.43/72.95 & 60.34/88.61 & 67.28/89.77 & 71.80/83.00 \\
OptFS       & 90.82/49.66 & 80.56/72.90 & 90.08/50.97 & 68.90/85.30 & 85.22/58.90 & 83.12/63.55 \\
LAPS        & 90.54/60.50 & \textbf{80.60}/77.78 & 90.32/54.17 & 68.47/86.88 & \textbf{86.15}/58.56 & 83.22/67.58 \\
\midrule
\multicolumn{7}{c}{\textit{Gradient-dependent}} \\
\midrule
AdaSCALE-A  & 86.10/54.83 & 79.23/\textbf{70.53} & 82.94/62.14 & 69.35/82.26 & 77.68/61.15 & 79.06/66.18 \\
AdaSCALE-L  & 85.14/61.89 & 79.05/71.96 & 82.33/67.01 & 69.29/83.19 & 76.74/63.58 & 78.51/69.53 \\
\midrule
\multicolumn{7}{c}{\textit{Training-free, gradient-free}} \\
\midrule
MSP         & 86.46/50.79 & 78.42/71.66 & 81.76/62.34 & \textbf{71.30}/\textbf{80.20} & 78.72/65.37 & 79.33/66.07 \\
MaxLogit    & 78.89/55.39 & 71.87/72.58 & 70.92/66.17 & 65.22/81.38 & 73.87/62.27 & 72.15/67.56 \\
Energy      & 67.79/73.19 & 64.28/79.82 & 60.40/76.75 & 59.27/86.00 & 69.28/66.29 & 64.20/76.41 \\
ASH-S       & 15.22/99.69 & 26.43/99.06 & 16.83/99.40 & 35.54/98.21 & 24.15/98.33 & 23.63/98.94 \\
SCALE       & 63.10/91.18 & 62.51/89.69 & 66.71/84.48 & 55.97/92.86 & 69.60/79.27 & 63.58/87.50 \\
LTS         & 84.82/52.99 & 75.66/72.54 & 79.49/63.65 & 65.41/82.51 & 78.18/62.94 & 76.71/66.93 \\
Anti-coh (Ours) & \textbf{91.66}/\textbf{45.54} & 80.24/71.74 & \textbf{90.71}/\textbf{48.42} & 67.72/85.20 & 86.14/\textbf{56.05} & \textbf{83.29}/\textbf{61.39} \\
\bottomrule
\end{tabular}%
}
\end{table}

\begin{table}[htbp]
\centering
\caption{OOD detection performance (AUROC/FPR@95, \%) on SwinV2-B (ImageNet) under the OpenOOD v1.5 protocol. Methods are grouped by resource requirement. The best value in each column is highlighted in \textbf{bold}.}
\label{tab:swinv2b_ablation}
\resizebox{\textwidth}{!}{%
\begin{tabular}{lcccccc}
\toprule
Method & iNaturalist & NINCO & OpenImage & SSB-Hard & Texture & Avg \\
\midrule
\multicolumn{7}{c}{\textit{Training-data-dependent}} \\
\midrule
ReAct       & 92.52/44.05 & \textbf{82.85}/68.88 & 90.33/49.21 & 72.77/81.37 & 85.87/53.87 & 84.87/59.48 \\
VRA         & 87.12/75.12 & 73.58/88.62 & 85.96/74.17 & 65.34/90.31 & 75.44/88.14 & 77.49/83.27 \\
OptFS       & 91.93/46.20 & 81.43/74.33 & 90.86/48.99 & 70.41/85.27 & 86.01/55.53 & 84.13/62.06 \\
LAPS        & 92.66/48.97 & 82.52/74.97 & 91.61/49.04 & 71.46/85.86 & \textbf{87.38}/\textbf{51.13} & \textbf{85.13}/61.99 \\
\midrule
\multicolumn{7}{c}{\textit{Gradient-dependent}} \\
\midrule
AdaSCALE-A  & 88.81/45.58 & 80.03/\textbf{66.43} & 85.61/54.45 & 72.85/78.39 & 78.45/57.34 & 81.15/60.44 \\
AdaSCALE-L  & 87.42/51.84 & 79.06/67.25 & 84.30/58.10 & 72.93/\textbf{77.65} & 76.99/58.94 & 80.14/62.76 \\
\midrule
\multicolumn{7}{c}{\textit{Training-free, gradient-free}} \\
\midrule
MSP         & 89.00/46.42 & 80.95/68.94 & 84.65/57.78 & \textbf{73.30}/78.81 & 80.83/62.54 & 81.75/62.90 \\
MaxLogit    & 85.44/46.96 & 77.31/68.07 & 78.76/57.84 & 70.16/78.30 & 78.91/58.26 & 78.12/61.89 \\
Energy      & 77.87/62.67 & 72.14/72.88 & 71.44/67.08 & 66.41/81.83 & 76.25/59.72 & 72.82/68.84 \\
ASH-S       & 12.06/99.73 & 23.61/98.55 & 13.75/99.63 & 30.08/98.61 & 22.58/99.11 & 20.41/99.13 \\
SCALE       & 50.23/90.69 & 49.15/91.66 & 48.04/89.70 & 44.58/95.98 & 64.20/78.99 & 51.24/89.40 \\
LTS         & 90.87/37.50 & 77.90/66.91 & 84.06/54.90 & 68.17/79.26 & 77.29/64.98 & 79.66/60.71 \\
Anti-coh (Ours) & \textbf{93.83}/\textbf{33.74} & 82.39/68.95 & \textbf{92.01}/\textbf{42.43} & 70.13/82.97 & 86.92/51.90 & 85.06/\textbf{56.00} \\
\bottomrule
\end{tabular}%
}
\end{table}

\begin{table}[htbp]
\centering
\caption{OOD detection performance (AUROC/FPR@95, \%) on ViT-B/16 (ImageNet) under the OpenOOD v1.5 protocol. Methods are grouped by resource requirement. The best value in each column is highlighted in \textbf{bold}.}
\label{tab:vitb_ablation}
\resizebox{\textwidth}{!}{%
\begin{tabular}{lcccccc}
\toprule
Method & iNaturalist & NINCO & OpenImage & SSB-Hard & Texture & Avg \\
\midrule
\multicolumn{7}{c}{\textit{Training-data-dependent}} \\
\midrule
ReAct       & 86.00/65.22 & 75.47/78.60 & 84.23/64.88 & 63.19/88.30 & 84.67/57.04 & 78.71/70.81 \\
VRA         & 17.86/98.99 & 42.69/92.89 & 42.53/90.37 & 37.52/96.73 & 66.39/73.21 & 41.40/90.44 \\
OptFS       & 89.92/56.08 & 80.51/74.63 & 88.22/60.09 & 66.18/87.22 & 86.44/56.86 & 82.25/66.98 \\
LAPS        & 87.21/68.73 & 78.06/78.46 & 86.48/64.77 & 63.63/88.88 & 85.86/56.67 & 80.25/71.50 \\
\midrule
\multicolumn{7}{c}{\textit{Gradient-dependent}} \\
\midrule
AdaSCALE-A  & 84.77/59.27 & 76.48/73.80 & 83.45/60.77 & 68.03/85.81 & 80.38/60.53 & 78.62/68.03 \\
AdaSCALE-L  & 79.84/74.55 & 75.30/76.56 & 80.76/67.06 & 66.87/87.34 & 76.56/66.47 & 75.87/74.40 \\
\midrule
\multicolumn{7}{c}{\textit{Training-free, gradient-free}} \\
\midrule
MSP         & 88.16/51.61 & 78.08/73.32 & 84.85/59.96 & \textbf{69.03}/\textbf{84.77} & 83.00/60.27 & 80.62/65.99 \\
MaxLogit    & 85.25/52.38 & 72.39/74.21 & 81.61/58.82 & 64.45/85.72 & 81.70/56.86 & 77.08/65.60 \\
Energy      & 79.26/64.06 & 66.05/78.45 & 76.49/65.05 & 59.16/88.16 & 79.32/58.46 & 72.05/70.83 \\
ASH-S       & 9.75/99.96 & 22.87/99.42 & 18.19/99.60 & 34.41/97.91 & 29.21/98.48 & 22.89/99.07 \\
SCALE       & 73.62/69.13 & 61.34/79.94 & 72.51/67.33 & 56.78/88.75 & 77.31/59.36 & 68.31/72.90 \\
LTS         & 88.89/\textbf{46.28} & 75.46/72.97 & 83.44/59.99 & 64.36/85.13 & 84.59/52.80 & 79.35/63.43 \\
Anti-coh (Ours) & \textbf{91.22}/49.08 & \textbf{80.99}/\textbf{72.83} & \textbf{89.13}/\textbf{56.73} & 65.10/87.06 & \textbf{88.16}/\textbf{50.20} & \textbf{82.92}/\textbf{63.18} \\
\bottomrule
\end{tabular}%
}
\end{table}

\begin{table}[htbp]
\centering
\caption{OOD detection performance (AUROC/FPR@95, \%) on ConvNeXt-B (ImageNet) under the OpenOOD v1.5 protocol. Methods are grouped by resource requirement. The best value in each column is highlighted in \textbf{bold}.}
\label{tab:convnextb_ablation}
\resizebox{\textwidth}{!}{%
\begin{tabular}{lcccccc}
\toprule
Method & iNaturalist & NINCO & OpenImage & SSB-Hard & Texture & Avg \\
\midrule
\multicolumn{7}{c}{\textit{Training-data-dependent}} \\
\midrule
ReAct       & 87.73/68.54 & 81.32/75.28 & 82.92/73.48 & 73.79/84.36 & 74.73/73.09 & 80.10/74.95 \\
VRA         & 74.22/78.71 & 68.00/85.20 & 75.69/74.51 & 55.73/91.36 & 71.83/74.68 & 69.10/80.89 \\
OptFS       & 92.81/41.73 & 84.04/68.48 & 91.07/48.80 & 73.90/82.44 & 85.66/55.59 & 85.50/59.41 \\
LAPS        & 91.67/60.68 & 83.30/75.25 & 90.35/59.19 & 73.45/87.89 & 85.03/58.10 & 84.76/68.22 \\
\midrule
\multicolumn{7}{c}{\textit{Gradient-dependent}} \\
\midrule
AdaSCALE-A  & 91.44/49.10 & 84.12/65.14 & 88.61/60.31 & 75.29/\textbf{75.84} & 78.80/68.63 & 83.65/63.81 \\
AdaSCALE-L  & 88.84/71.72 & 83.41/70.74 & 87.89/70.22 & 74.42/79.19 & 78.52/75.04 & 82.61/73.38 \\
\midrule
\multicolumn{7}{c}{\textit{Training-free, gradient-free}} \\
\midrule
MSP         & 89.51/44.25 & 81.07/66.49 & 83.15/59.26 & \textbf{75.38}/76.11 & 72.53/67.62 & 80.33/62.75 \\
MaxLogit    & 82.68/57.41 & 76.58/69.39 & 73.50/68.81 & 72.90/76.86 & 63.33/73.39 & 73.80/69.17 \\
Energy      & 55.97/90.24 & 64.43/82.68 & 52.95/89.38 & 65.67/84.36 & 49.12/88.40 & 57.63/87.01 \\
ASH-S       & 4.49/99.76 & 29.22/97.16 & 22.87/97.61 & 33.64/97.85 & 42.78/92.87 & 26.60/97.05 \\
SCALE       & 57.62/98.73 & 67.77/90.66 & 65.35/92.05 & 61.35/92.97 & 69.53/86.10 & 64.32/92.10 \\
LTS         & 91.64/38.85 & 82.52/\textbf{62.25} & 83.71/60.16 & 73.04/77.31 & 70.96/71.56 & 80.37/62.02 \\
Anti-coh (Ours) & \textbf{93.66}/\textbf{37.43} & \textbf{84.56}/68.10 & \textbf{91.98}/\textbf{46.65} & 73.87/82.23 & \textbf{87.53}/\textbf{52.25} & \textbf{86.32}/\textbf{57.33} \\
\bottomrule
\end{tabular}%
}
\end{table}

\begin{table}[htbp]
\centering
\caption{OOD detection performance (AUROC/FPR@95, \%) on DeiT3-B (ImageNet) under the OpenOOD v1.5 protocol. Methods are grouped by resource requirement. The best value in each column is highlighted in \textbf{bold}.}
\label{tab:deit3b_ablation}
\resizebox{\textwidth}{!}{%
\begin{tabular}{lcccccc}
\toprule
Method & iNaturalist & NINCO & OpenImage & SSB-Hard & Texture & Avg \\
\midrule
\multicolumn{7}{c}{\textit{Training-data-dependent}} \\
\midrule
ReAct       & 79.41/80.53 & 74.77/84.62 & 76.89/81.40 & 63.84/90.18 & 78.73/74.20 & 74.73/82.19 \\
VRA         & 33.04/97.80 & 45.76/91.94 & 42.29/93.09 & 46.84/94.63 & 38.88/96.84 & 41.36/94.86 \\
OptFS       & 91.61/43.64 & 82.07/70.70 & 89.09/53.03 & 69.01/84.26 & 84.77/57.78 & 83.31/61.89 \\
LAPS        & 92.06/50.68 & \textbf{82.40}/73.48 & \textbf{89.46}/59.56 & 69.45/87.38 & 85.66/61.83 & \textbf{83.81}/66.59 \\
\midrule
\multicolumn{7}{c}{\textit{Gradient-dependent}} \\
\midrule
AdaSCALE-A  & 90.03/45.91 & 80.81/70.48 & 88.27/52.23 & 72.07/81.86 & 82.74/54.80 & 82.78/61.06 \\
AdaSCALE-L  & 89.78/49.95 & 80.04/74.80 & 87.09/63.58 & \textbf{72.45}/83.53 & 81.81/55.69 & 82.23/65.51 \\
\midrule
\multicolumn{7}{c}{\textit{Training-free, gradient-free}} \\
\midrule
MSP         & 91.24/41.57 & 81.33/70.18 & 87.19/54.03 & 72.37/\textbf{80.68} & 82.91/58.78 & 83.01/61.05 \\
MaxLogit    & 90.02/43.84 & 79.08/71.96 & 85.04/56.68 & 71.00/82.04 & 82.29/55.99 & 81.49/62.10 \\
Energy      & 84.76/67.37 & 74.91/82.65 & 79.63/74.06 & 68.70/87.96 & 80.56/63.05 & 77.71/75.02 \\
ASH-S       & 53.54/87.52 & 52.45/91.14 & 52.96/88.84 & 56.72/92.23 & 45.63/93.62 & 52.26/90.67 \\
SCALE       & 87.45/53.05 & 77.59/75.48 & 85.27/59.08 & 66.36/86.23 & \textbf{85.77}/51.60 & 80.49/65.09 \\
LTS         & 90.71/41.36 & 79.40/71.35 & 85.71/54.88 & 71.00/81.93 & 83.41/54.22 & 82.04/60.75 \\
Anti-coh (Ours) & \textbf{92.34}/\textbf{35.53} & 80.43/\textbf{69.79} & 88.01/\textbf{50.52} & 71.09/81.49 & 85.35/\textbf{48.74} & 83.44/\textbf{57.21} \\
\bottomrule
\end{tabular}%
}
\end{table}

\begin{table}[htbp]
\centering
\caption{OOD detection performance (AUROC/FPR@95, \%) on EfficientNet-B1 (ImageNet) under the OpenOOD v1.5 protocol. Methods are grouped by resource requirement. The best value in each column is highlighted in \textbf{bold}.}
\label{tab:efficientnetb1_ablation}
\resizebox{\textwidth}{!}{%
\begin{tabular}{lcccccc}
\toprule
Method & iNaturalist & NINCO & OpenImage & SSB-Hard & Texture & Avg \\
\midrule
\multicolumn{7}{c}{\textit{Training-data-dependent}} \\
\midrule
ReAct       & 85.19/82.48 & 73.82/86.66 & 84.57/79.68 & 61.54/91.80 & 83.41/86.19 & 77.71/85.36 \\
VRA         & 84.33/74.80 & 72.38/81.08 & 84.97/68.08 & 61.32/89.98 & 88.64/59.29 & 78.33/74.65 \\
OptFS       & \textbf{91.37}/\textbf{44.46} & \textbf{79.01}/74.79 & \textbf{89.66}/50.99 & 65.78/87.45 & \textbf{92.30}/35.62 & \textbf{83.62}/58.66 \\
LAPS        & 86.54/69.70 & 74.68/82.19 & 85.95/67.41 & 63.96/89.22 & 85.91/63.97 & 79.41/74.50 \\
\midrule
\multicolumn{7}{c}{\textit{Gradient-dependent}} \\
\midrule
AdaSCALE-A  & 83.37/59.15 & 74.24/75.77 & 84.67/59.07 & \textbf{71.45}/\textbf{81.55} & 79.20/65.66 & 78.59/68.24 \\
AdaSCALE-L  & 77.65/79.36 & 69.77/85.23 & 81.70/72.57 & 69.58/85.68 & 77.47/76.35 & 75.23/79.84 \\
\midrule
\multicolumn{7}{c}{\textit{Training-free, gradient-free}} \\
\midrule
MSP         & 85.86/55.77 & 78.60/\textbf{73.78} & 83.35/62.40 & 70.71/82.93 & 78.82/66.67 & 79.47/68.31 \\
MaxLogit    & 80.91/63.73 & 73.28/76.51 & 78.82/67.34 & 67.58/83.28 & 71.31/72.96 & 74.38/72.76 \\
Energy      & 70.54/87.57 & 66.12/87.77 & 70.99/83.52 & 63.51/88.10 & 61.84/89.45 & 66.60/87.28 \\
ASH-S       & 45.07/99.83 & 41.72/99.69 & 46.01/99.08 & 47.91/97.93 & 58.99/96.44 & 47.94/98.59 \\
SCALE       & 49.45/99.59 & 43.47/99.52 & 49.71/98.48 & 48.47/97.87 & 62.85/93.87 & 50.79/97.87 \\
LTS         & 68.76/85.75 & 58.82/90.88 & 67.44/83.51 & 58.54/90.55 & 68.06/83.21 & 64.32/86.78 \\
Anti-coh (Ours) & 90.11/45.42 & 78.51/75.08 & 89.33/\textbf{48.16} & 62.83/88.19 & 91.84/\textbf{30.55} & 82.52/\textbf{57.48} \\
\bottomrule
\end{tabular}%
}
\end{table}

\begin{table}[htbp]
\centering
\caption{OOD detection performance (AUROC/FPR@95, \%) on MobileNetV2 (ImageNet) under the OpenOOD v1.5 protocol. Methods are grouped by resource requirement. The best value in each column is highlighted in \textbf{bold}.}
\label{tab:mobilenetv2_ablation}
\resizebox{\textwidth}{!}{%
\begin{tabular}{lcccccc}
\toprule
Method & iNaturalist & NINCO & OpenImage & SSB-Hard & Texture & Avg \\
\midrule
\multicolumn{7}{c}{\textit{Training-data-dependent}} \\
\midrule
ReAct       & 92.43/45.20 & 77.46/78.96 & 89.58/55.08 & 66.43/87.01 & 90.88/41.03 & 83.36/61.45 \\
VRA         & 92.20/45.18 & 77.57/77.27 & 90.90/47.96 & 64.78/87.52 & 94.59/24.27 & 84.01/56.44 \\
OptFS       & 91.78/45.68 & 78.07/79.28 & 89.86/54.54 & 66.74/87.06 & 93.45/33.97 & 83.98/60.11 \\
LAPS        & \textbf{96.07}/\textbf{22.03} & 78.03/70.43 & 90.26/41.01 & \textbf{73.52}/\textbf{74.66} & 94.13/28.60 & 86.40/47.35 \\
\midrule
\multicolumn{7}{c}{\textit{Gradient-dependent}} \\
\midrule
AdaSCALE-A  & 90.89/50.45 & 77.65/77.29 & 91.26/47.43 & 70.99/84.40 & 91.63/37.80 & 84.48/59.47 \\
AdaSCALE-L  & 92.23/43.87 & 78.68/74.63 & 92.29/40.68 & 71.19/83.17 & 94.09/26.63 & 85.70/53.80 \\
\midrule
\multicolumn{7}{c}{\textit{Training-free, gradient-free}} \\
\midrule
MSP         & 86.74/59.55 & 77.96/81.69 & 82.07/72.69 & 69.27/87.94 & 78.97/70.89 & 79.00/74.55 \\
MaxLogit    & 90.68/53.82 & 77.00/81.42 & 87.11/65.24 & 67.14/87.94 & 86.12/57.30 & 81.61/69.15 \\
Energy      & 90.37/55.18 & 76.12/81.80 & 87.08/64.53 & 66.45/88.24 & 86.61/54.24 & 81.33/68.80 \\
ASH-S       & 92.63/38.67 & 78.83/73.05 & 91.30/42.03 & 68.14/84.34 & 96.43/15.60 & 85.47/50.74 \\
SCALE       & 94.43/30.29 & 79.97/70.98 & 92.15/38.49 & 70.11/82.76 & 96.66/14.63 & 86.66/47.43 \\
LTS         & 94.41/30.51 & 80.02/71.54 & 92.14/38.96 & 70.04/82.96 & 96.59/14.61 & 86.64/47.72 \\
Anti-coh (Ours) & 94.98/27.66& \textbf{80.82}/\textbf{69.87} & \textbf{92.70}/\textbf{36.83} & 71.00/81.63 & \textbf{96.87}/\textbf{14.52} & \textbf{87.27}/\textbf{46.10} \\
\bottomrule
\end{tabular}%
}
\end{table}

\begin{table}[htbp]
\centering
\caption{OOD detection performance (AUROC/FPR@95, \%) on BEiTv2-B (ImageNet) under the OpenOOD v1.5 protocol. Methods are grouped by resource requirement. The best value in each column is highlighted in \textbf{bold}.}
\label{tab:beitv2b_ablation}
\resizebox{\textwidth}{!}{%
\begin{tabular}{lcccccc}
\toprule
Method & iNaturalist & NINCO & OpenImage & SSB-Hard & Texture & Avg \\
\midrule
\multicolumn{7}{c}{\textit{Training-data-dependent}} \\
\midrule
ReAct       & 93.53/32.77 & 82.07/67.85 & 90.64/43.59 & 70.36/81.80 & 88.09/46.44 & 84.94/54.49 \\
VRA         & 28.30/98.05 & 38.13/97.45 & 35.32/96.40 & 41.23/95.97 & 51.79/92.48 & 38.95/96.07 \\
OptFS       & 94.19/32.36 & \textbf{83.79}/68.20 & 92.68/38.49 & 69.98/83.04 & 88.41/47.52 & 85.81/53.92 \\
LAPS        & 94.23/35.21 & 83.60/70.64 & 92.66/40.87 & 70.21/83.63 & \textbf{89.52}/43.51 & \textbf{86.04}/54.77 \\
\midrule
\multicolumn{7}{c}{\textit{Gradient-dependent}} \\
\midrule
AdaSCALE-A  & 90.45/37.23 & 82.15/\textbf{62.47} & 89.87/42.70 & \textbf{75.34}/\textbf{75.79} & 80.74/51.81 & 83.71/54.00 \\
AdaSCALE-L  & 88.02/45.18 & 80.92/64.73 & 89.03/46.60 & 74.85/76.13 & 78.23/55.53 & 82.21/57.63 \\
\midrule
\multicolumn{7}{c}{\textit{Training-free, gradient-free}} \\
\midrule
MSP         & 90.80/38.29 & 81.93/65.46 & 87.98/48.13 & 73.32/77.89 & 85.03/52.32 & 83.81/56.42 \\
MaxLogit    & 89.48/35.20 & 78.40/63.73 & 85.04/45.69 & 70.57/77.37 & 84.06/47.85 & 81.51/53.97 \\
Energy      & 86.88/42.58 & 74.76/69.28 & 81.68/52.76 & 67.64/80.40 & 82.65/50.12 & 78.72/59.03 \\
ASH-S       & 9.93/99.86 & 22.89/99.35 & 12.17/99.61 & 35.01/97.81 & 20.63/99.10 & 20.13/99.15 \\
SCALE       & 72.30/66.23 & 62.86/79.94 & 64.38/74.96 & 60.79/84.78 & 73.88/63.85 & 66.84/73.95 \\
LTS         & 87.08/39.80 & 74.46/67.20 & 80.72/52.57 & 67.32/79.26 & 81.57/50.34 & 78.23/57.83 \\
Anti-coh (Ours) & \textbf{94.36}/\textbf{30.18} & 83.52/66.74 & \textbf{93.06}/\textbf{37.11} & 69.61/82.08 & 89.43/\textbf{43.00} & 86.00/\textbf{51.82} \\
\bottomrule
\end{tabular}%
}
\end{table}

\pagebreak
\section{Limitation}
\label{appendix:limitation}
Despite requiring no training data, our method still trails the strongest representation-based methods on backbones with negative penultimate activations, as discussed in Appendix~\ref{appendix:Rep_method}. Relying solely on the activation and classifier weight of a single test sample, our method cannot exploit the class-conditional statistics that representation-based methods estimate from the training set, and this may impose a fundamental ceiling on how far a training-free method can close this gap. Moreover, our method underperforms several baselines on CIFAR-100 where most methods show a substantial drop in absolute performance relative to CIFAR-10. We suspect that the limited number of training images per class prevents the pretrained model from learning sufficiently discriminative features and thereby limits the reliability of the evaluation on this dataset (see Appendix~\ref{appendix:full_results}). In addition, our method shows a lower performance on SSB-Hard, and addressing this on fine-grained near-OOD datasets is a promising direction for future work.
\end{document}